\documentclass[10pt,review,conference]{IEEEtran}
\IEEEoverridecommandlockouts

\usepackage[framemethod=TikZ]{mdframed}
\mdfsetup{nobreak=false}

\usepackage{xspace}
\usepackage{colortbl}
\usepackage{amsmath,amsfonts,amssymb}
\usepackage{algorithm}
\usepackage{algpseudocode}
\usepackage{graphicx}
\usepackage{textcomp}
\usepackage{xcolor}
\usepackage{tikz}
\usepackage{svg}
\usepackage[most]{tcolorbox}
\usepackage{graphbox}
\usepackage{mathtools}
\usepackage{fancybox}
\usepackage{makecell}
\usepackage{multirow}
\usepackage{booktabs}
\usepackage{hyperref}
\hypersetup{hidelinks}
\usepackage{cleveref}
\usepackage[normalem]{ulem} \usepackage{listings}
\usepackage{caption}
\usepackage{subcaption}
\usepackage{longtable}
\usepackage{pifont}\usepackage{enumitem}
\usepackage{subfloat}
\usepackage[latin1]{inputenx}
\usepackage[T1]{fontenc}
\usepackage{soul}
\usepackage{courier}
\usepackage{svg}
\usepackage[switch]{lineno}

\newboolean{showcomments}
\setboolean{showcomments}{false}
\ifthenelse{\boolean{showcomments}}
{\newcommand{\nb}[2]{
  \fcolorbox{black}{yellow}{\bfseries\sffamily\scriptsize#1}
  {\sf\small$\blacktriangleright$\textit{#2}$\blacktriangleleft$}
 }
 
}
{\newcommand{\nb}[2]{}
 
}

\newcommand{\Mnist}{MNIST\xspace}
\newcommand{\Taxinet}{TaxiNet\xspace}
\newcommand{\Yolo}{YOLOv4\xspace}
\newcommand{\fga}{FGA\xspace}
\newcommand{\efga}{EFGA\xspace}
\newcommand{\Lymphoma}{LSC\xspace}

\newcounter{commentnumber}

\newcommand\referenceaddressed[1]{\ifthenelse{\equal{\ref{#1}}{\ref{#1end}}}{page~\pageref{#1} line~\ref{#1}}{page~\pageref{#1} lines~\ref{#1}---\ref{#1end}}}

\lstdefinestyle{mystyle}{
    commentstyle=\color{gray},
    keywordstyle=\bfseries,
    numberstyle=\tiny\color{darkgray},
    basicstyle=\linespread{1.1}\ttfamily\footnotesize,
    breakatwhitespace=false,
    breaklines=true,
    captionpos=b,
    keepspaces=true,
    numbers=left,
    numbersep=10pt,
    showspaces=false,
    showstringspaces=false,
    showtabs=false,
    tabsize=4,
    frame=single,
}

\usepackage[framemethod=TikZ]{mdframed}
\newenvironment{Answer}[1][]{\ifstrempty{#1}{\mdfsetup{frametitle={\tikz[baseline=(current bounding box.east),outer sep=0pt]
      \node[line width=0pt,anchor=east,rectangle,draw=white,fill=white]
    ;}}
  }{\mdfsetup{frametitle={\tikz[baseline=(current bounding box.east),outer sep=0pt]
      \node[anchor=east,rectangle,draw=white,fill=white]
    {\strut #1};}}}\mdfsetup{innertopmargin=-5pt,linecolor=black,linewidth=0.5pt,topline=true,frametitleaboveskip=\dimexpr-\ht\strutbox\relax,}
  \begin{mdframed}[]\relax }{\end{mdframed}}
\newcommand\phase[1]{\tikz[baseline=(X.base)]\node [draw=black,fill=white,thick,rectangle,inner sep=2pt, rounded corners=2pt](X){\color{black}\textbf{#1}};}

\usepackage[english]{babel} 	\usepackage{amsthm}		

\theoremstyle{plain}
\newtheorem{theorem}{Theorem}
\newtheorem*{remark}{Remark}

\theoremstyle{definition}

\makeatletter
\newcommand{\linebreakand}{\end{@IEEEauthorhalign}
  	\hfill\mbox{}\par
  	\mbox{}\hfill\begin{@IEEEauthorhalign}
}
\makeatother

\def\BibTeX{{\rm B\kern-.05em{\sc i\kern-.025em b}\kern-.08em
    T\kern-.1667em\lower.7ex\hbox{E}\kern-.125emX}}

\begin{document}
\title{Ensembles-based Feature Guided Analysis}

\author{\IEEEauthorblockN{\href{https://orcid.org/0000-0002-3033-7371}{Federico Formica}}
\IEEEauthorblockA{McMaster University, Canada\\
formicaf@mcmaster.ca}
\and
\IEEEauthorblockN{\href{https://orcid.org/0009-0008-8494-0329}{Stefano Gregis}}
\IEEEauthorblockA{University of Bergamo, Italy\\
s.gregis4@studenti.unibg.it}
\and
\IEEEauthorblockN{\href{https://orcid.org/0009-0008-1648-4130}{Andrea Rota}}
\IEEEauthorblockA{University of Bergamo, Italy\\
a.rota51@studenti.unibg.it}
\linebreakand
\IEEEauthorblockN{\href{https://orcid.org/0009-0008-0655-8335}{Aurora Francesca Zanenga},}
\IEEEauthorblockA{University of Bergamo, Italy\\
aurora.zanenga@unibg.it}
\and
\IEEEauthorblockN{\href{https://orcid.org/0000-0003-3161-2176}{Mark Lawford}}
\IEEEauthorblockA{McMaster University, Canada\\
lawford@mcmaster.ca}
\and
\IEEEauthorblockN{\href{https://orcid.org/0000-0001-5303-8481}{Claudio Menghi}}
\IEEEauthorblockA{University of Bergamo, Italy\\
McMaster University, Canada\\
claudio.menghi@unibg.it}
}

\maketitle  

\thispagestyle{plain}
\pagestyle{plain}

\begin{abstract}
Recent Deep Neural Networks (DNN) applications ask for techniques that can explain their behavior.
Existing solutions, such as Feature Guided Analysis (FGA), extract rules on their internal behaviors, e.g., by providing explanations related to neurons activation.
Results from the literature show that these rules have considerable precision (i.e., they correctly predict certain classes of features), but the recall (i.e., the number of situations these rule apply) is more limited.
To mitigate this problem, this paper presents Ensembles-based Feature Guided Analysis (EFGA). 
EFGA combines rules extracted by \fga into ensembles.
Ensembles aggregate different rules to increase their applicability depending on an aggregation criterion, a policy that dictates how to combine rules into ensembles.
Although our solution is extensible, and different aggregation criteria can be developed by users, in this work, we considered three different aggregation criteria.
We evaluated how the choice of the criterion influences the effectiveness of EFGA on two benchmarks (i.e., the \Mnist and \Lymphoma datasets), and found that different aggregation criteria offer alternative trade-offs between precision and recall.
We then compare \efga with \fga. 
For this experiment,  we selected an aggregation criterion that provides a reasonable trade-off between precision and recall.
Our results show that EFGA has higher train recall (+28.51\% on \Mnist, +33.15\% on \Lymphoma), and test recall (+25.76\% on \Mnist, +30.81\% on \Lymphoma) than FGA, with a negligible reduction on the test precision (-0.89\% on \Mnist, -0.69\% on \Lymphoma).

 \end{abstract}

\section{Introduction}
Artificial Intelligence techniques, such as Deep Neural Networks (DNN), are helping humans in several tasks and activities (e.g.,~\cite{khan2022software,boujida2024neural}). 
DNN is a specific artificial intelligence technique that, among the other tasks, can identify whether some input data belongs to a class or not.
For example, DNN can check for a pathology in a patient's radiography.

DNN behavior derives from data~\cite{1634649}.
Provided with a set of annotated images, a neural network learns how to classify the images.
For example, given a set of patients' radiographies labeled with the presence (or absence) of a pathology, the DNN learns how to classify new images.
This approach differs from the one used for traditional software, which requires engineers to define the software behavior explicitly.
This paradigm complicates the interpretation of the software behavior~\cite{baier2019challenges} since it is difficult to determine how the neural network decides whether a certain image belongs to a class or not.
These images constitute a high-dimensional input, which makes it even harder to understand how a DNN can discern high-level concepts, such as the presence or absence of a specific pathology. For example, the IDC (Invasive Ductal Carcinoma) dataset~\cite{janowczyk2016} uses $50\times 50$ images and each pixel can assume $266$ values for each baseline color (Red, Green, Blue). This leads settings lead to an input space of dimension $[0,255]^{50\times50\times3}$.
To mitigate this problem, the software engineering community is interested techniques that can motivate the actions selected by the neural network~\cite{molnar2022interpretable}.

Recent works (e.g.,~\cite{kim2018,yeh2020,kusters2020,ghorbani2019,koh2020,chen2020concept,barbiero2022entropy,Gopinath_2023}) tried to extracts rules related to the internals of the neural networks.
Feature-Guided Analysis (\fga)~\cite{Gopinath_2023} is one of these approaches.
It extracts rules describing how the behaviors of certain neurons of a DNN influence the detection the presence (or absence) of some human understandable concepts.
\fga rules consist of a precondition and a postcondition.
The precondition  is a condition (some of) the neurons of the neural network values, while
the postcondition asserts the feature presence (or absence).
The rules extracted by FGA entail that the postcondition holds when the condition specified by the precondition is satisfied.
For example, a rule can explain engineers that a specific pathology whenever some neurons are active.

\fga was originally evaluated on the \Taxinet~\cite{Beland_2020, Frew_2004},  \Yolo-Tiny~\cite{caesar2020nuscenes} benchmarks.
A recent study~\cite{FGA_Replication} replicated the experiments on \Mnist~ \cite{lecun1998} and \Lymphoma~\cite{janowczyk2016} benchmarks and assessed the capability of the extracted rules to verify the presence of certain visual features.
\fga was evaluated by considering the precision and recall of the extracted rules.
Precision concerns the ``correctness'' of the rules: If a rule has high precision, it is highly likely that, when an image satisfies the precondition, then the postcondition (feature presence or absence) is the correct one.
Recall concerns the ``applicability'' of the rule: The higher the recall, the more applicable it is. 
For example, a rule with high recall will detect a higher presence of inputs with the digit ''1'' than a rule with lower recall, which will detect the feature in less images. 
This rule can be extremely precise, but its applicability limited to the subset of images where the feature was identified.
For example, on a recent example (M-DNN1), \fga has shown remarkable results in terms of precision on the \Taxinet (min=0\%, max=100\%, avg=76.64\%, std=34.69\%), \Yolo-Tiny (min=69\%, max=91\%, avg=74.25\%, std=7.25\%), and \Mnist (min=98.66\%, max=100\%, avg=99.63\%, std=0.38\%) benchmarks.
However, the recall of the rules is more limited: The average recall for the \Taxinet, \Yolo-Tiny, and \Mnist benchmarks is (min=0\%, max=100\%,avg=61.88\%, std=37.11\%), (min=20\%, max=59\%, avg=32.13\%, std=11.96\%), and  (min=28.73\%, max=90.17\%, avg=60.3\%, std=20.16\%) respectively.
Although the \fga replication study demonstrated that \fga is capable of generating rules with high precision, a limited recall can be a limit in several applications. 
For example, a small recall provide an explanation for the presence of pathology that refers to a limited number of cases and does not provide the engineers with an exhaustive description for the detection of a given pathology.

To address this limitation, this paper presents Ensembles-based Feature Guided Analysis (EFGA). 
EFGA combines rules produced by \fga into ensembles, aggregations of rules that increase their applicability.
It is parametherized via an aggregation criteria, a policy that dictates how to combine rules into ensembles.
In this work, we propose three different aggregation criteria that respectively aggregate the best X rules by train recall (\texttt{TOP(X)}), iteratively increase the size of the ensemble until it reaches a threshold on the train recall for the training dataset (\texttt{REC(X)}), and aggregates all the rules with train recall above the average of the list of rules obtained for that feature (\texttt{AVG}).
We evaluated (RQ1) how the choice of the criterion used to build the ensemble influences the effectiveness of EFGA. 
Our results show that different solutions offer variable trade-offs between precision and recall, and that \texttt{TOP(10)} reaches a reasonable compromise between these two metrics.
We then compare EFGA with \texttt{TOP(10)} as aggregation criteria and compare it with \fga. 
Our results show that \efga has higher train recall (+33.15\% for L-DNN1), and test recall (+30.81\% for L-DNN1) with a negligible reduction on the  test precision (-0.69\%) compared to those obtained for \fga.

To summarize, our contributions are as follows:
\begin{itemize}
    \item A novel technique (i.e., Ensembles-based Feature Guided Analysis) that extends \fga to increase its recall;
    \item An empirical analysis on how the selection of the aggregation criteria (and its parameters) influence  the effectiveness of the solution; 
    \item A rigorous comparison of EFGA and  \fga;
    \item A systematic discussion on our results and their threats to validity;
    \item A complete replication package containing the dataset, implementation of EFGA, replication scripts, and results.
\end{itemize}

This paper is organized as follows.
\Cref{sec:background} summarizes Feature-Guided Analysis.
\Cref{sec:efga_algo} presents Ensembles-based Feature Guided Analysis (EFGA).
\Cref{sec:implementation}  describes implementation details.
\Cref{sec:eval} evaluates our contribution.
\Cref{sec:discussion} discusses our results and threats to validity.
\Cref{sec:related} summarizes related work.
\Cref{sec:conclusion} concludes. 
 \section{Background}\label{sec:background}
\Cref{sec:benchmark} briefly introduces \Mnist and \Lymphoma benchmarks.
\Cref{sec:fga} summarizes Feature Guided Analysis.
\Cref{sec:precision} defines their precision and recall.

\subsection{Benchmark}
\label{sec:benchmark}
Our benchmark consists of the MNIST (Modified National Institute of Standards and Technology database)~\cite{lecun1998} and \Lymphoma (Lymphoma Subtype Classification)~\cite{janowczyk2016} datasets.

\Mnist is a dataset of handwritten gray-scale images representing digits. 
Each image is associated with a label indicating the corresponding digit. The dataset contains 70'000 images: 60'000 images representing the training dataset, and 10'000 images representing the test dataset.

\Lymphoma is a dataset from the National Institute on Aging. 
It contains a collection of histopathological images for the classification of three lymphoma types: Chronic Lymphocytic Leukemia (CLL), Follicular Lymphoma (FL), and Mantle Cell Lymphoma (MCL).
The dataset contains 374 images (113 CLL, 139 FL, and 122 MCL).
Following the approach of the original paper, each image ($1388\times1040$ px) was cropped into 1376 overlapping patches of $36\times36$ px with a stride of 32, yielding a total of 514'624 patches.
The original work proposed a winner-take-all decision logic where the DNN returned the classification for each patch, and the most frequent class became the classification of the entire image.
The original dataset provides the ground-truth classification for the original 374 images, and not for the individual patches, since some patches may not contain evidence of any lymphoma type.
To solve this problem, we decided to filter the dataset of 514'624 patches and considered only the ones for which the network proposed in~\cite{janowczyk2016} returned a classification score above $95\%$.
Since this network achieves a high classification accuracy (96.58\% $\pm$ 0.01\%~\cite{janowczyk2016}), this ensures that only the patches that can be confidently classified in one of the three classes are considered.
This filtering process reduced the dataset to 442'398 patches, out of which 135'574 were classified as CLL, 169'367 as FL, and 137'457 as MCL.

\subsection{Feature Guided Analysis}\label{sec:fga}
\fga~\cite{Gopinath_2023} extracts a representation of a visual feature from a feedforward neural network. This type of DNN is organized in layers, each one with neurons producing an output based on the output of the neurons at the previous layer.
\Cref{fig:approachFGA} shows the \fga approach.  
\fga consists of the \emph{Extract Neurons Activations} (\phase{1}) and \emph{Compute Decision Tree} (\phase{2}) phases.

 \emph{Extract Neurons Activations} (\phase{1}). 
    It computes a dataset (\texttt{D}$^\prime$) from the DNN model (\texttt{M}), the layers (\texttt{L}) to be considered by \fga, a dataset (\texttt{D}) of images, and a set of features of interest (\texttt{F}), 
The dataset (\texttt{D}$^\prime$) contains the activation values of all the neurons from the layers (\texttt{L}) of the model (\texttt{M}) for each input image in the dataset (\texttt{D}) and a label (for every feature) indicating whether the feature from \texttt{F} is present or absent in the input images. \Cref{tab:dataset_example} shows an example of dataset \texttt{D}$^\prime$ structured like the ones of FGA. Activation contains the activation values for each Neuron ($\text{N}_{1,1}$...$\text{N}_{1,n}$), while the columns Digit 1 and Line show the presence or absence of two example features that we defined, where Line aggregates the digits 1,4 and 7 while Digit 1 is the feature describing the presence of the digit 1.

\begin{table}[t]
    \centering
    \caption{Example of dataset \texttt{D}$^\prime$ structured like the ones of FGA. }
    \begin{tabular}{l l l }
        \toprule
        \textbf{Activation} & \textbf{Digit 1}    &\textbf{Line}\\
        \midrule
         $[0.34,-1.23,0.17,...]$ & 1 & 0\\
         $[0.42,-0.87,0.55,...]$ & 0 & 1\\
$\vdots$ & $\vdots\,$ & $\vdots\,$ \\
         $[3.12,3.11,1.31,...]$ & 1 & 0\\
        \bottomrule
    \end{tabular}
    \label{tab:dataset_example}
\end{table}

\begin{figure}[t]
    \centering
       \tikzstyle{output} = [coordinate]

\begin{tikzpicture}[auto,
 block/.style ={rectangle, draw=black, thick, fill=white!20, text width=8em,align=center, rounded corners},
 block1/.style ={rectangle, draw=blue, thick, fill=blue!20, text width=5em,align=center, rounded corners, minimum height=2em},
 line/.style ={draw, thick, -latex',shorten >=2pt,scale=0.65},
 cloud/.style ={draw=red, thick, ellipse,fill=red!20,minimum height=1em,scale=0.65},scale=0.65]

\node [output] (INITNode) {};
\draw (0,0) node[block,right of=INITNode, node distance=0cm] (B1) {\phase{1}  Extract Neurons Activations};
\node [output, left of=B1, node distance=3.2cm] (NN) {};
\node [block, right of=B1, node distance=3.6cm] (B2) {\phase{2} Compute Decision Tree};
\node [output, right of=B2, node distance=2cm] (End) {};

\draw[-stealth] (NN.south) -- (B1)
    node[midway]{\texttt{M}, \texttt{L}, \texttt{D}, \texttt{F}};
\draw[-stealth] (B1) -- (B2)
    node[midway,above,align=right]{ \texttt{D'}};
\draw[-stealth] (B2) -- (End)
    node[midway,above]{\texttt{R}};
    
 \end{tikzpicture}     \caption{Overview of Feature Guided Analysis.}
    \label{fig:approachFGA}
\end{figure}
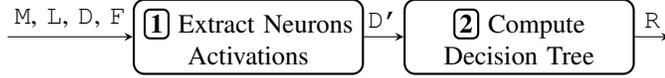

 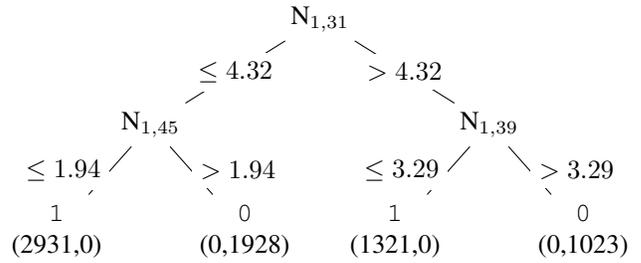
\begin{figure}[t]
    \centering
       \begin{tikzpicture}[level 1/.style={sibling distance=45mm},level 2/.style={sibling distance=25mm},level 3/.style={sibling distance=10mm}, level distance=40pt
]
\node[]{$\text{N}_{1,31}$}
child {
    node[] {$\text{N}_{1,45}$}
        child{
            node[align=center]{\texttt{1} \\ (2931,0)} edge from parent node[fill=white,left] {$\leq 1.94$}
        }
        child{
            node[align=center]{\texttt{0} \\ (0,1928)} edge from parent node[fill=white,right] {$> 1.94$}
        }
        edge from parent node[fill=white] {$\leq 4.32$}
}
child { 
         node[align=center]{$\text{N}_{1,39}$}
         child{
                node[align = center]{\texttt{1} \\ (1321,0)} edge from parent node[fill=white,left] {$\leq 3.29$}
            }
          child{
            node[align=center]{\texttt{0} \\ (0,1023)} edge from parent node[fill=white,right] {$> 3.29$}
          }
            edge from parent node[fill=white] {$> 4.32$}
    };
\end{tikzpicture}
      \caption{Example of decision tree structured like the ones from \fga and \efga.}
    \label{fig:tree}
\end{figure}

\lstset{numbers=left,xleftmargin=5em,frame=single,framexleftmargin=3em,tabsize=1}
\begin{figure}[t]
\centering
\begin{lstlisting}[caption = {Pseudocode for Feature Guided Analysis.}, label = {alg:pseudocode_fga}]
function FGA(M,D,L,F){*\label{line:main_fga}*
	labels = assignFeature(M,D,F)*\label{line:assignfeature}*
	activations = M.activations(L,D)*\label{line:get_act}*
	for feature in F and layer in L:{ *\label{line:for_loop}*
		activation = activations.get(layer) *\label{line:get_act_layer}*
		$\text{D'}_{tr}$ = [activation.tr,labels.tr]*\label{line:data_train}*
		tree = DecisionTree($\text{D'}_{tr}$)*\label{line:tree}*
		paths = extractPaths(tree) *\label{line:paths}*
		$\text{D'}_{te}$ = [activation.test,labels.test]*\label{line:data_test}*
		R.add(evPt(paths,layer,feature,$\text{D'}_{tr}$,$$$\text{D'}_{te}$))*\label{line:list}*
	}
	return R *\label{line:return_fga}*
}
\end{lstlisting}
\end{figure}

\emph{Compute Decision Tree} (\phase{2}): For every feature in \texttt{F}, the dataset (\texttt{D}$^\prime$) is used to compute a decision tree.
For example, the decision tree from \Cref{fig:tree} shows an example of decision tree structured like the ones built by \fga.
The decision tree has three layers, one root node and four leaf nodes. 
The decision tree define conditions on the values assumed by the neurons of the DNN model that entail the presence or absence of the feature.
Each node of the tree refers to a neuron, i.e., each node is represented as $\text{N}_{x,y}$ where $x$ is the layer and $y$ is the neuron number.
Edges are labeled with conditions related to the activation value of the neuron.
For example, the edge from $\text{N}_{1,31}$ to $\text{N}_{1,45}$ considers the activations of the neuron $\text{N}_{1,31}$ lower or equal than 4.32. 
Leaf nodes are associated to the label describing the presence~(1) of the feature in the input images, while the others to the one describing feature absence~(0).
Specifically, each leaf node is associated to a tuple of values $(a,b)$ where $a$ is the number of dataset inputs with label describing the feature presence and $b$ is the number of inputs describing feature absence. 
Pure leaf nodes (e.g. leaves with all inputs belonging to the same label) indicate cluster of images that (do not) possess a certain feature.
All the spurious nodes in the tree are discarded, so we consider only pure leaf nodes.
A path from the root node of the tree to a pure leaf node defines an assertion made by the conjunction of all the conditions of the edges related to the path.
This condition entails the presence (or absence) of the feature.
For example, if the activation of $\text{N}_{1,31}$ is lower or equal than 4.32 and the one of $\text{N}_{1,45}$ is lower or equal than 1.94, the decision tree predicts the presence of the feature in the input image.
This assertion and the information related to the presence (or absence) of the feature are combined in a decision rule of the form \textbf{pre}~$\rightarrow$~\textbf{post}, where \textbf{pre} is the assertion on the values of some neurons and post indicates  the presence (or absence) of the feature. 
Therefore, \fga will extract from this decision tree the following rules with postcondition describing the feature presence: 
$(\text{N}_{1,31} \leq 4.32 \land \text{N}_{1,45} \leq 1.94) \rightarrow \texttt{1}$ and
$(\text{N}_{1,31} > 4.32 \land \text{N}_{1,39} \leq 3.29) \rightarrow \texttt{1}$. 
Note that for every feature, a separate Decision Tree is computed. 
This choice enables to assign to multiple features to every input. 
For example, if one feature describes the presence of the digit ``1'' and a second feature describes the presence of a straight line in the digits ``1'', ``4'', ``7'', then an image of the digit ``1'' has both the first and the second feature.

FGA returns a set ($R$) of decision rules as output.
We define the length of a rule as the number of clauses in its precondition. 
For our example, the rules have length two since they have two clauses joined in a conjunction in their precondition.

\Cref{alg:pseudocode_fga} reports the pseudo code for the algorithm of \fga.
\Cref{line:main_fga} shows the body of the \fga function.
\Cref{line:assignfeature} assigns the labels to the input images in \texttt{D} for every feature in \texttt{F}. \Cref{line:get_act} extracts the activation values from the DNN.  \Cref{line:for_loop} loops on every feature and  every layer to extract the activation values at a layer (\Cref{line:get_act_layer}), and builds the training dataset (\Cref{line:data_train}). \Cref{line:tree} estimates the decision tree on the training dataset, and \Cref{line:paths} extracts the paths from the root node to the pure leaf nodes of the tree. Note that in our experiments, the decision trees had only pure leaf nodes, so all the input images were considered when defining the rules. This may not be true for all DNNs, and some of the input images may not be considered to formulate the rules since the spurious nodes are discarded by our implementation. Then, the test dataset is computed (\Cref{line:data_test}) and the extracted paths are evaluated on both training and test datasets.
The function returns the list of all rules, at all layers, for all the features considered. 
In this work, for consistency with~\cite{Gopinath_2023}, we updated the implementation from~\cite{FGA_Replication} to consider misclassified inputs.
In their replication study~\cite{FGA_Replication}, the authors did not consider misclassified inputs and extracted the rules from correctly classified images. Considering the high accuracy of the neuronal networks used, the fraction of misclassified inputs is small compared to the size of the \Mnist dataset. 

\subsection{Precision and Recall}

\label{sec:precision}
We selected precision and recall as metrics to evaluate the performances of the rules, as done in \cite{Gopinath_2023}. Intuitively
precision is defined as $\frac{\text{TP}}{\text{TP} + \text{FP}}$, while recall as $\frac{\text{TP}}{\text{TP} + \text{FN}}$.
Precision measures of how confident we are that the rule is correct about the presence or absence of a feature in the input image, while recall describes the percentage of inputs with (or without) the feature captured by the rule considered.
We define a True Positive input (TP) as an input that has the feature that satisfies the rule, while we define as False Positive (FP) an input that does not have the feature but satisfies the rule.
We defined a False Negative input (FN) as an input that, despite showing the visual feature, does not satisfy the rule. 
Let us consider again the rules for the postcondition ''1'' extracted from the decision tree in \Cref{fig:tree}. The test dataset from which the tree is extracted has 7203 inputs, 4000 with the feature and 3203 without the feature. If the rule $(\text{N}_{1,31} \leq 4.32 \land \text{N}_{1,45} \leq 1.94) \rightarrow \texttt{1}$ has 2900 TP inputs and 31 FP, it will have a precision of 98.94\%. 
If the rule $(\text{N}_{1,31} > 4.32 \land \text{N}_{1,39} \leq 3.29) \rightarrow \texttt{1}$ has 1100 TP inputs and 221 FP, its precision will be 83.27\%.
In addition, by considering 2900 TP and 1100 FN for the first rule, 1100 TP and 2900 FN for the second rule, the recall is respectively 72.5\% and 27.5\%. \section{Ensembles-based Feature Guided Analysis}
\label{sec:efga_algo}

We extended the \fga approach (\Cref{sec:background}) to produce ensembles of rules.
\Cref{fig:approach} provides an overview of the phases of the algorithm computing the ensembles of rules. 
The first two phases (\phase{1} and \phase{2}) correspond to the standard FGA approach (see \Cref{sec:background}).
The last phase (\phase{3}) creates the ensemble of rules by considering a criterion ($C$) specified as input and the decision rules ($R$). The criterion defines which rules will be considered for the ensemble. 
Unlike \fga, which returns a set of rules in the form \textbf{pre}~$\rightarrow$~\textbf{post},  
\efga a set of ensembles (of rules).
Each ensemble has the form  $(\textbf{pre}_1\rightarrow$~\textbf{post})$\vee(\textbf{pre}_2\rightarrow$~\textbf{post} $)\vee \dots \vee (\textbf{pre}_n\rightarrow$~\textbf{post}$)$.\footnote{Or equivalently $(\textbf{pre}_1 \vee \textbf{pre}_2 \vee \dots \vee \textbf{pre}_n)\rightarrow$~\textbf{post}.} 
Intuitively, the postcondition (\textbf{post}) holds whenever one of the preconditions ($\textbf{pre}_1$, $\textbf{pre}_2$, \ldots, $\textbf{pre}_n$) is satisfied.

To compute the ensemble, EFGA creates the disjunction of several rules of the form ($\textbf{pre}\rightarrow\textbf{post}$) extracted by FGA.
These rules are extracted based on the aggregation criterion ($C$).
Consider our example from \Cref{fig:tree}. 
EFGA aggregates the rules 
 $(\text{N}_{1,31} \leq 4.32 \land \text{N}_{1,45} \leq 1.94) \rightarrow \texttt{1}$ and
 $(\text{N}_{1,31} > 4.32 \land \text{N}_{1,39} \leq 3.29) \rightarrow \texttt{1}$ 
 into the  ensemble:
  $(\text{N}_{1,31} \leq 4.32 \land \text{N}_{1,45} \leq 1.94) \rightarrow \texttt{1} \vee (\text{N}_{1,31} > 4.32 \land \text{N}_{1,39} \leq 3.29) \rightarrow \texttt{1}$ 
Therefore, according to this defition of ensemble, we can update the definitions of TP, FP and FN.
We define a TP input as an input that has the feature that satisfies one of the rules of the ensemble, while we define as FP an input that does not have the feature but satisfies one of the rules of the ensemble.
Lastly, we defined a FN input as an input that, despite showing the visual feature, does not satisfy any of the rules in the ensemble. 
 Therefore, the ensemble categorizes the presence or absence of a feature based on the satisfaction of the preconditions of one of the individual rules. 
 Since these rules are extracted from a decision tree structure, each input can satisfy only one rule, since it can satisfy only one precondition. 
 This property ensures that, when aggregating multiple rules into an ensemble, the number of inputs satisfying the ensemble of rules always increases. These considerations support the following theorem for the Training and Test metrics of \efga.

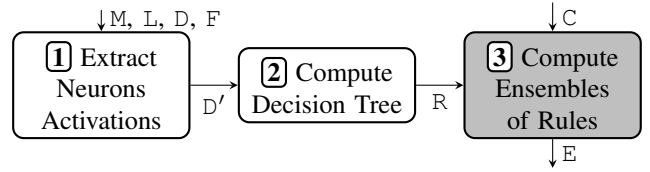
\begin{figure}
    \centering
     \tikzstyle{output} = [coordinate]

\begin{tikzpicture}[auto,
 block/.style ={rectangle, draw=black, thick, fill=white!20, text width=6em,align=center, rounded corners},
 block1/.style ={rectangle, draw=blue, thick, fill=blue!20, text width=5em,align=center, rounded corners, minimum height=2em},
 line/.style ={draw, thick, -latex',shorten >=2pt,scale=0.65},
 cloud/.style ={draw=red, thick, ellipse,fill=red!20,minimum height=1em,scale=0.65},scale=0.65]

\node [output] (INITNode) {};
\draw (0,0) node[block,right of=INITNode, node distance=0cm] (B1) {\phase{1}  Extract Neurons Activations};
\node [output, above of=B1, node distance=1cm] (NN) {};
\node [block, right of=B1, node distance=3cm] (B2) {\phase{2} Compute Decision Tree};
\node [block, right of=B2, node distance=3cm, fill = lightgray] (B3) {\phase{3} Compute Ensembles of Rules};
\node [output, below of=B3, node distance=1.1cm] (End) {};
\node [output, above of=B3, node distance=1.1cm] (CRIT) {};

\draw[-stealth] (NN.south) -- (B1.north)
    node[midway]{\texttt{M}, \texttt{L}, \texttt{D}, \texttt{F}};
\draw[-stealth] (B1.east) -- (B2.west)
    node[midway,below]{\texttt{\texttt{D}$^\prime$}};
\draw[-stealth] (B2.east) -- (B3.west)
    node[midway,below,]{\texttt{R}};
\draw[-stealth] (B3.south) -- (End.north)
    node[midway]{\texttt{E}};
\draw[-stealth] (CRIT.south) -- (B3.north)
node[midway]{\texttt{C}};
 \end{tikzpicture}     \caption{Ensembles-based Feature Guided Analysis.}
    \label{fig:approach}
\end{figure}

\begin{theorem}
\label{theor:recallComp}
Let $R$ be a list of $n$ rules of the form \textbf{pre}$_i$~$\rightarrow$~\textbf{post}$_i$,  $\mathit{Re}(R_{i})$ be the recall of the rule $R_{i}$, 
and $E$ be the ensemble defined as $E\coloneqq \bigvee_{i=1}^{n}$\textbf{pre}$_i$~$\rightarrow$~\textbf{post}$_i$.
The recall $\mathit{Re}(E)$ of $E$ equals  the sum of the recalls of the rules in R:
\begin{equation}
	\mathit{Re}(E) = \sum_{i=1}^{n} \mathit{Re}(R_{i})
	\label{eq:recallSum}
\end{equation}
\end{theorem} 

\begin{proof}[Proof Sketch]
The recall of the rule $R_{i}$ (see \Cref{sec:precision}) is defined as
\begin{equation}
	\mathit{Re}(R_{i}) = \frac{\mathit{TP}_{i}}{\mathit{TP}_{i}+\mathit{FN}_{i}}
	\label{eq:recallDef}
\end{equation}
An ensemble considers rules that share the same post-condition (i.e., they all detect the same feature).
Note that all the rules are applied on the same dataset $\mathcal{D}$.
Let us identify as $P$ the total number of ground-truth positive (i.e., the number of images in the dataset that possess a feature).
Each rule can have a different number of $\mathit{TP}$ and $\mathit{FN}$, but the sum of these two values will always be equal to the overall number $P$ of inputs showing the feature.
Therefore: 
\begin{equation}
	P = \mathit{TP}_{i}+\mathit{FN}_{i}	\qquad\forall i : 1 \leq i \leq n
	\label{eq:DenomEqual}
\end{equation}
Note that since the rules are extracted from the leaf nodes of a decision tree, the preconditions of two different rules are pairwise disjoint (i.e., an image in $\mathcal{D}$ can satisfy at most the precondition of one rule).
Therefore, the number of True Positives for the composite rule is equal to the algebraic sum of the True Positives for each rule:
\begin{equation}
	\mathit{TP}_{c} = \sum_{i=1}^{n} \mathit{TP}_{i}
	\label{eq:sumTP}
\end{equation}

Therefore, the recall of the ensemble is:
\begin{align*}
	\mathit{Re}(E) 	&= \frac{\mathit{TP}_{c}}{P}
\end{align*}

To conclude, the recall of the ensemble equals the sum of the recalls of its rules.
\end{proof}

\begin{remark}
The same rule does not hold for the precision, since the number of true and false positives is not constant among different rules for the same feature.
\end{remark}

To compute the ensembles we propose three criteria:
\begin{enumerate}
    \item \texttt{TOP(X)}: Aggregates the best X rules (by train recall).
    \item \texttt{REC(X)}: Iteratively increases the size of the ensemble until it has train recall above X\%.
    \item \texttt{AVG}: Aggregates all the rules with train recall above the average of the list of rules obtained for that feature.
\end{enumerate}
While \texttt{TOP} and \texttt{REC} require the user to set a threshold value, \texttt{AVG} automatically selects the rules that will be part of the ensemble. As the rules are listed in descending order (by train recall), \texttt{TOP} aggregates the X rules that together guarantee the highest coverage of input images (showing the feature at the selected layer). This does not guarantee a specific level of recall, which can be reached by using \texttt{REC}. 
\texttt{REC(X)} increases the size of the ensemble until a minimum coverage level X of the input images showing the feature is reached.  \section{Implementation}
\label{sec:implementation}
\lstset{numbers=left,xleftmargin=4em,frame=single,framexleftmargin=3em, xrightmargin=1em}
\begin{figure}[t]
\centering
\begin{lstlisting}[caption = {Ensembles-based Feature Guided Analysis.}, label = {alg:pseudocode}]
function main(M, D, L, F, C) = { *\label{line:main}*
	rulesList = featureGuidedAnalysis(M,D,L,F)*\label{line:fga}*
	ensemblesList = buildEns(rulesList,C,$\text{D}_{tr}$)*\label{line:buildEnsemble}*
	E = evaluateEns(ensemblesList,$\text{D}_{te}$)*\label{line:evaluate}*
	return E *\label{line:return}*
}
\end{lstlisting}
\end{figure}

We implemented \efga as a plugin for \fga~\cite{Gopinath_2023}.
We consider the implementation from~\cite{FGA_Replication} since it is publicly available.
\Cref{alg:pseudocode} reports the pseudo code for the algorithm \efga, which uses the \fga implementation from \cref{sec:background}.
\Cref{line:main} defines the main function. The inputs are the trained neural network (\texttt{M}), the \Mnist or \Lymphoma dataset (\texttt{D}), the layers from which the rules will be extracted (\texttt{L}), the features chosen by the user (\texttt{F}) and the criterion required to build the ensemble of rules (\texttt{C}).
\Cref{line:fga}  calls the \fga function by passing as parameters the model (\texttt{M}), the dataset (\texttt{D}), the layers (\texttt{L}), and the features selected (\texttt{F}).
\Cref{line:buildEnsemble} aggregates the rules, sorted in descending order by number of inputs in the training dataset ($\texttt{D}_{tr}$) satisfying the rule. The rules are added to the ensemble starting with the one offering the highest train recall, until the desired criterion is matched. \Cref{line:evaluate} evaluates the ensembles of rules (\texttt{E}) by computing the Metrics on the test dataset ($\texttt{D}_{te}$). 
\Cref{line:return} returns the ensemble of rules (\texttt{E}).

 \section{Evaluation}\label{sec:eval}
We considered the following research questions:
\begin{itemize}
    \item \textbf{Influence of the Criterion Choice (RQ1)}. How does the choice of the criterion used to build the ensemble influence the effectiveness of \efga? (\Cref{sec:rq1})
    \item \textbf{Effectiveness of \efga (RQ2).} How effective is \efga compared with \fga? (\Cref{sec:rq2})
\end{itemize}

This first research question assesses the influence of the criterion to build the ensemble on the effectiveness of \efga. 
The second research question assesses the effectiveness of \efga by comparing its performance with \fga, which relies on the same technology but does not rely on ensembles.

\emph{Benchmarks}.
For the \Mnist dataset, we consider the neural network architecture (M-DNN1) proposed by the authors of Prophecy \cite{Gopinath_2019}, but not included in the original publication. The network has $10$ layers, with $2$ convolutional and $2$ dense layers with 128 and 10 neurons.
The activations of the convolutional and dense layers are considered as separate layers.
We did not train the neural network for this experiment, but we downloaded a pretrained version.
For the two dense layers, we respectively used a ReLU and a softmax activation function.

For the \Lymphoma dataset, we utilized the L-DNN1 and L-DNN2 architectures~\cite{janowczyk2016}.
The 12-layer CNN consists of a feature extractor component build of three sequential Convolution-ReLU-Pooling blocks, a classification section made by two fully-connected layers, and a \texttt{SoftmaxWithLoss} layer to compute the class probabilities.
L-DNN2, is an upgraded version of L-DNN1 with two ReLU-Dropout blocks after each of the two fully-connected layers.
We downloaded the pretrained Caffe model~\cite{lymphomacaffe2016} for the first 5-fold cross-validation split and converted it to the ONNX for use in our framework.

\subsection{Influence of the Criterion Choice (RQ1)}\label{sec:rq1}
To assess how the criterion used to build the ensemble affects the effectiveness of \efga, we proceeded as follows.

\begin{figure*}[ht!]
\centering
 \begin{subfigure}[b]{0.40\textwidth}
    \centering
       \includegraphics[width=1\linewidth]{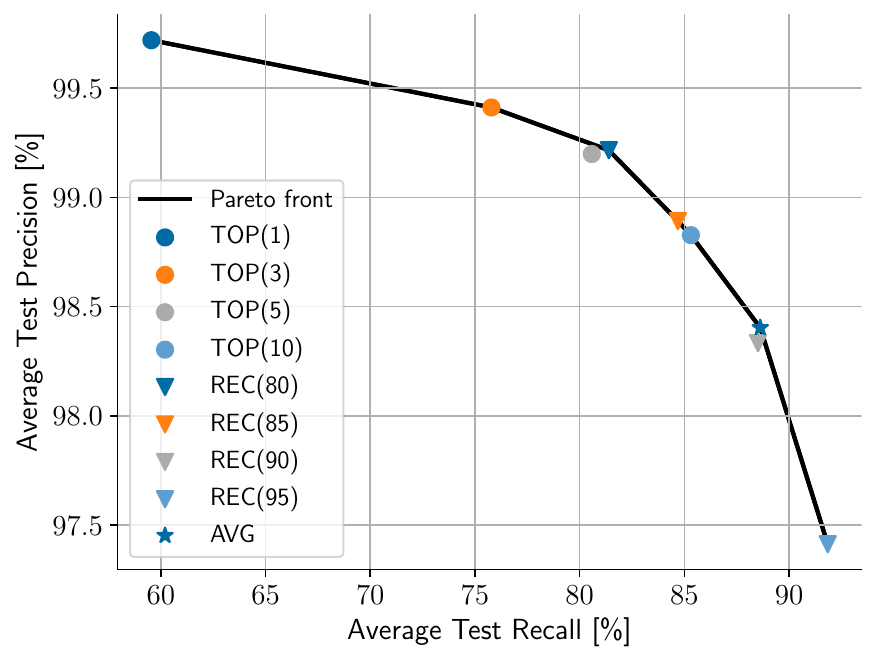}
    \caption{Test recall and precision for M-DNN1.}
    \label{fig:pareto-precision-m-dnn1}
\end{subfigure}
\hspace{10mm}
\begin{subfigure}[b]{0.40\textwidth}
    \centering
       \includegraphics[width=1\linewidth]{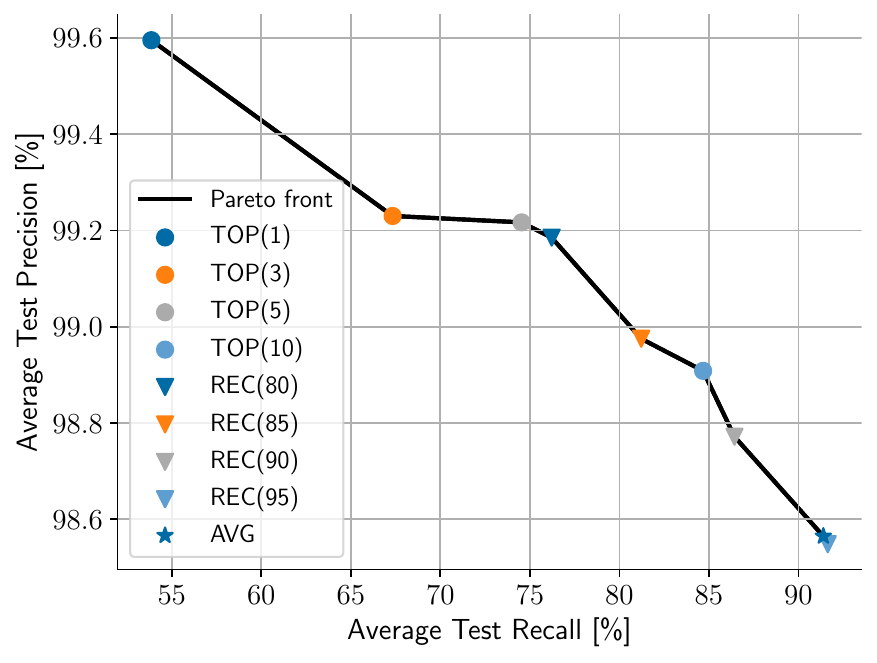}
    \caption{Test recall and precision for  L-DNN1.}
    \label{fig:pareto-precision-l-dnn1}
\end{subfigure}
\\
\begin{subfigure}[b]{0.40\textwidth}
    \centering
       \includegraphics[width=1\linewidth]{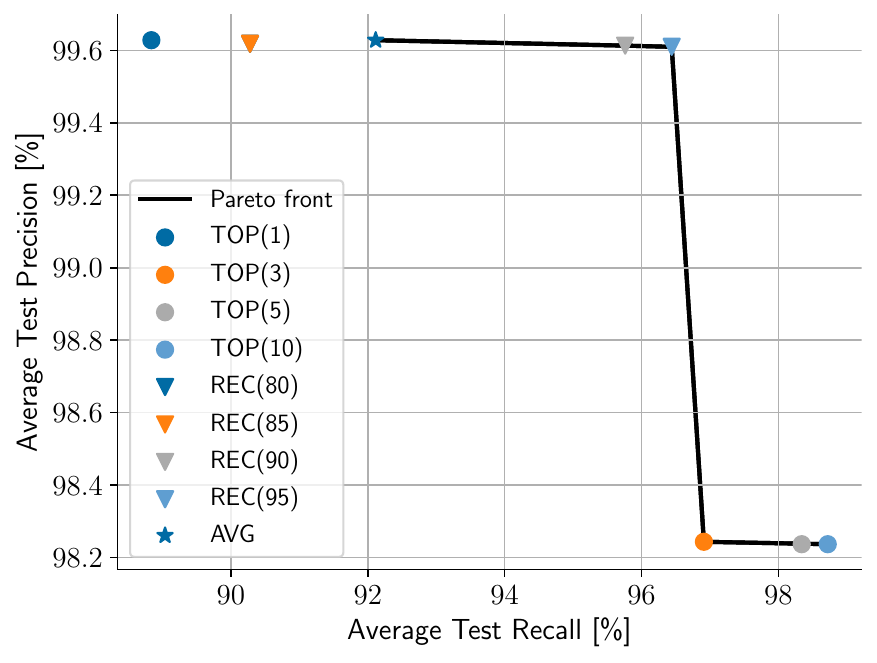}
    \caption{Test recall and precision L-DNN2. \phantom{A few words to add an extra line.}}
    \label{fig:pareto-precision-l-dnn2}
\end{subfigure}
\hspace{10mm}
\begin{subfigure}[b]{0.40\textwidth}
    \centering
       \includegraphics[width=1\linewidth]{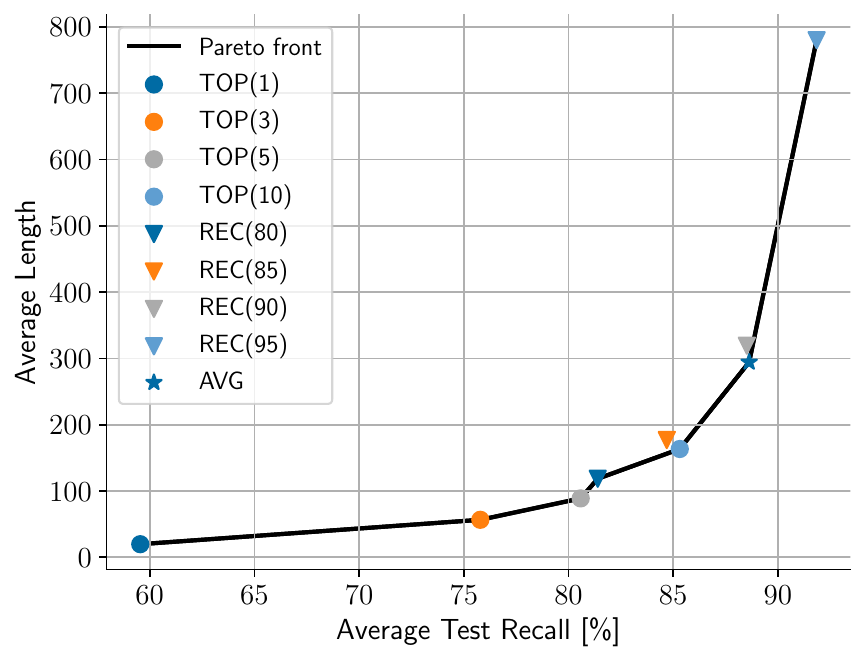}
    \caption{Test Recall and Length of Ensembles of Rules for M-DNN1.}
    \label{fig:pareto-len-m-dnn-1}
\end{subfigure}
\\
\begin{subfigure}[b]{0.40\textwidth}
    \centering
       \includegraphics[width=1\linewidth]{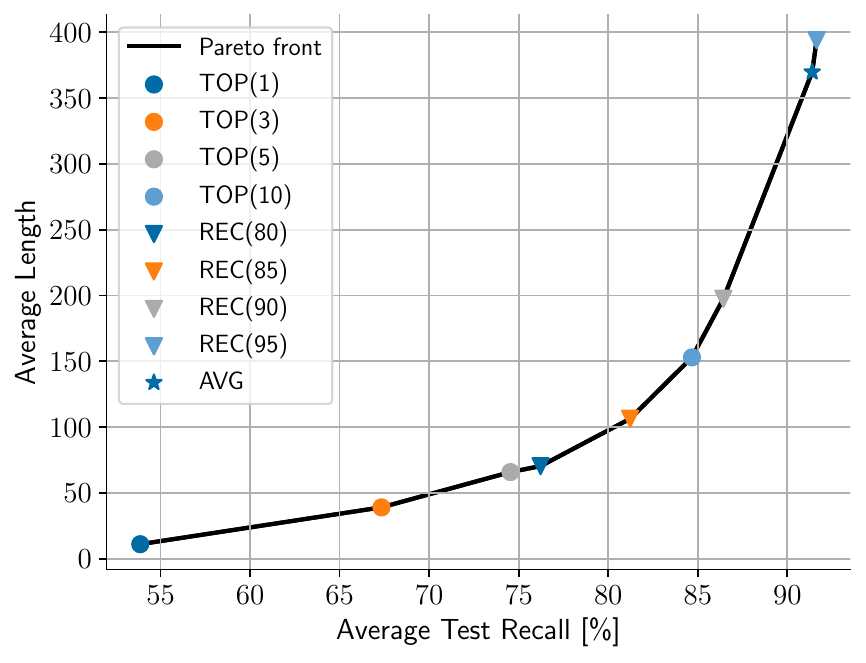}
    \caption{Test Recall and Length of Ensembles of Rules for L-DNN1.}
    \label{fig:pareto-len-l-dnn-1}
\end{subfigure}
\hspace{10mm}
\begin{subfigure}[b]{0.40\textwidth}
    \centering
       \includegraphics[width=1\linewidth]{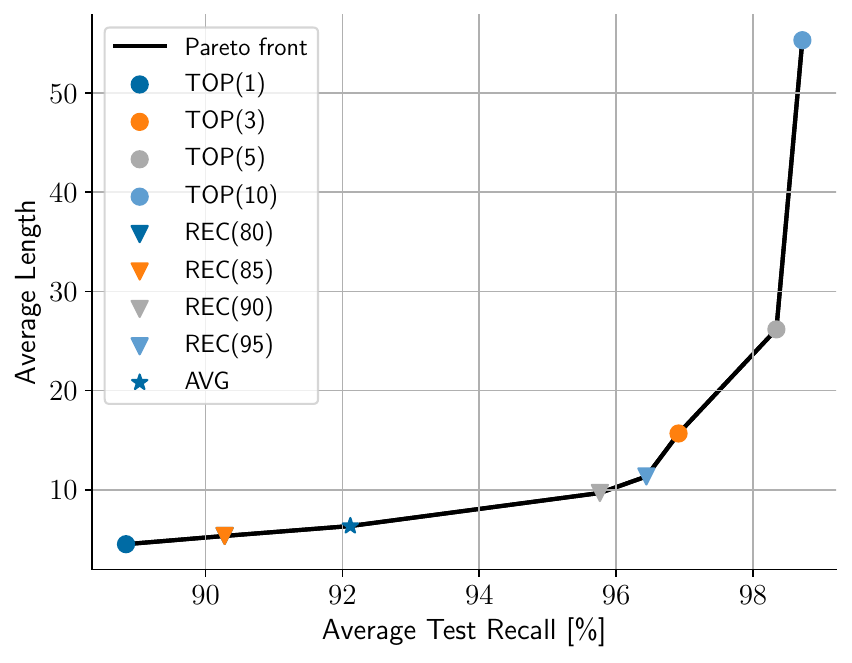}
    \caption{Test Recall and Length of Ensembles of Rules for L-DNN2.}
    \label{fig:pareto-len-l-dnn-2}
\end{subfigure}
\caption{Comparison of different \efga criteria.}
\label{fig:RQ1}
\end{figure*}

\textbf{Methodology}.
For \Mnist, we select the only hidden dense layer of the network (i.e., the neurons from the first dense layer, after the activation of the first dense layer, and before the activation of the last dense layer) to be considered for extracting the \fga rules.
We considered the features from the \fga replication study \cite{FGA_Replication} reported in \Cref{tab:mnist-features}.
These features include the presence of single digits ``0'', ``1'', $\ldots$, ``9'', couples of digits graphically similar ``2'' and ``7'', and ``9'' and ``6'', sets of digits sharing a similar characteristics such as the presence of a straight line (``1'', ``4'', and ``7'') or a circle (``0'', ``6'', ``8'', and ``9'').
Since the features we considered are directly related to classification labels used by the MNIST dataset, we could label the presence or absence of a feature automatically.

For \Lymphoma, we select the neurons from two dense layers to be considered by \fga and \efga since they are the only dense layers in the network.
For the features, we considered the three individual classes ``CLL'', ``FL'', and ``MCL'', and their pairwise combinations ``CLL \& FL'', ``CLL \& MCL'', and ``FL \& MCL''.

\begin{table}[t]
    \centering
    \caption{Name, set of digits, and description of the feature.}
    \begin{tabular}{l r l }
        \toprule
        \textbf{Feature}    &\textbf{Digits} & \textbf{Description} \\
        \midrule
         Digit 0 & 0 & Input image shows digit ``0''.\\
         \quad$\vdots$ & $\vdots\,$\\
         Digit 9 & 9 & Input image shows digit ``9''.\\
         2 and 7 & 2,7 & Digits graphically similar.\\
         9 and 6 & 9,6 & Digits graphically similar.\\
         Line & 1,4,7 & Four digits with a vertical line.\\
         Circle & 0,6,8,9 & Four digits with a circular shape.\\
        \bottomrule
    \end{tabular}
    \label{tab:mnist-features}
\end{table}

For every layer, we save all the rules extracted (e.g., all the paths from the root node to a pure leaf node of the tree). 
As only pure leaves are considered in the extraction process, all rules obtained have 100\% training precision.
For every feature, we created the ensemble using the list of rules of a single layer. We chose the layer with the rule with the highest train recall for the feature under consideration (i.e., the single rule that is satisfied by the highest number of images in the dataset).

For \texttt{TOP(X)}, we computed the ensemble aggregating 1, 3, 5, and 10 rules. 
Note that the ensemble produced by \texttt{TOP(1)} corresponds to the rule produced by \fga. 
We considered 80\%, 85\%, 90\%, and 95\% as threshold values for \texttt{REC(X)},

For every ensemble, we computed train and test precision and recall. 
Furthermore, we computed the length (the sum of the lengths of the rules in the ensemble) of each ensemble. 
For every criterion and threshold value, we computed the average metrics over all the features.

\textbf{Results}.
\Cref{fig:RQ1} presents scatter plots for each model containing the results for RQ1. The markers (dot, triangle, and star) represent the average value across all features. 
\Cref{fig:pareto-precision-m-dnn1,fig:pareto-precision-l-dnn1,fig:pareto-precision-l-dnn2} report average test recall on the x-axis and average test precision on the y-axis. 
The line represents the Pareto front.
\Cref{fig:pareto-len-m-dnn-1,fig:pareto-len-l-dnn-1,fig:pareto-len-l-dnn-2} report average test recall on the x-axis and average length of the ensemble on the y-axis. 

For M-DNN1, \Cref{fig:pareto-precision-m-dnn1} shows that the ensembles created with the criterion \texttt{REC(90)} and \texttt{TOP(5)} are not optimal, since other solutions from the Pareto front that dominate these solutions (e.g., \texttt{REC(80)} offers an  average test recall and precision higher than \texttt{TOP(5)} and \texttt{REC(90)}).
The figure also shows that the test precision decreases as the test recall increases. 
\texttt{REC(95)} is the solution with the highest recall.
It offers a significantly higher (+32.3\%) recall than \texttt{TOP(1)} with a negligible (-2.31\%) reduction of the precision. 
\Cref{fig:pareto-len-m-dnn-1} shows that the ensembles created with the criterion \texttt{REC(85)} and \texttt{REC(90)} are dominated by other solutions (\texttt{TOP(10)}) that offer a lower  average lengths for the rules and a higher test recall.
The figure also shows that the length of the ensemble increases with the test recall. 
\texttt{REC(95)} is the solution with the highest length and recall.
It offers a significantly higher (+32.3\%) recall than \texttt{TOP(1)}, but its length is considerably higher (+759.92).
Together, the two scatter plots from \Cref{fig:pareto-len-m-dnn-1,fig:pareto-precision-m-dnn1} show that increasing the recall comes with a cost: Longer ensembles.
The choice of the best solution from the Pareto front is application-specific. 
It depends on the application domain, desired recall, and length of the rules. 
\texttt{TOP(10)} offers a reasonable compromise: it reaches a recall higher than 85\% and generates ensembles with a length lower than 200.

For L-DNN1 (\Cref{fig:pareto-precision-l-dnn1}), we observe a trend similar to M-DNN1: \texttt{TOP(1)} shows limited recall (53.86\%). The criteria \texttt{REC(95)} and AVG achieve the highest recall (>90\%) but with a reduction in precision. \texttt{TOP(10)} represents a balanced trade-off, achieving 84.68\% recall with precision close to 99\%.
As shown in \Cref{fig:pareto-len-l-dnn-1}, L-DNN1 shows a trade-off similar to M-DNN1. While TOP(1) yields short ensembles, increasing the recall requires significantly more complex rules.

For L-DNN2 (\Cref{fig:pareto-precision-l-dnn2}), the base model (\texttt{TOP(1)}) already exhibits high recall (89.02\%) and precision (99.63\%). Increasing the ensemble size with \texttt{REC(95)} improves recall to 96.44\% while maintaining high precision. However, criteria like \texttt{TOP(10)} push recall to 98.72\% but suffer a drop in precision (98.24\%).
L-DNN2 exhibits much more compact ensembles than M-DNN1 and L-DNN1 (\Cref{fig:pareto-len-l-dnn-2}).
Also the base TOP(1) rules (length 4.50) achieve good recall. The criterion REC(95) boosts recall with a very modest length of 11.33. The largest ensembles generated by TOP(10) reach 98.72\% recall with a length of 55.33, remaining significantly shorter than those for M-DNN1 or L-DNN1.

\begin{Answer}[RQ1---Influence of the Criterion]
The results for M-DNN1 and L-DNN1 demonstrate that the choice of the \efga criterion can increase the recall to 37.77\% with a negligible (-1.05\%) reduction of the precision. 
However, increasing the recall comes with a cost: Longer ensembles.
To reach a recall of 95\%, the length of the ensemble increased from 19.79 to 779.71.
TOP(10) offers a reasonable compromise between rule length and recall across the datasets, keeping the length below 200 in most cases.
For L-DNN2, the results show that for this DNN a limited number of rules achieve high recall and adding further rules does not significantly improve performance. 
\end{Answer} \subsection{Effectiveness of \efga (RQ2)}
\label{sec:rq2}
To answer the second research question, we compare the best rules extracted from \efga (configured with the we select the \texttt{TOP(10)} ensemble from \Cref{sec:rq1}) and \fga. 

\textbf{Methodology}.
To compare \efga and \fga, we compute the ensemble using the same methodology, layers, and features from \Cref{sec:rq1} and  performed experiments \emph{Exp1} and \emph{Exp2}.

 \emph{Exp1}: We compare \efga and \fga when applied to M-DNN1, L-DNN1 and L-DNN2. Note that M-DNN1 was considered over other DNNs since it has the lowest train and test recall; the metrics \efga aims to improve.

\emph{Exp2}: We also compare the results of \efga from M-DNN1 to the results produced by \fga applied to M-DNN3, an additional network architecture used in \cite{FGA_Replication}.
    This network has the best results in terms of training and test recall.
    We did not consider this network for RQ1 since results would not be informative: The average train and test recall provided by \fga over the considered features are already above 84\%.
    However, in this experiment, we assess whether \efga enables reaching results comparable to those reached by \fga on other datasets where it is effective.
    Therefore, we compare the results obtained with \fga over M-DNN1 with those of \efga over M-DNN3.

We present the results of each experiment separately.

\begin{table*}[t]
    \centering
    \caption{Train ($R_{tr}$) and Test ($R_{te}$) Recall and the Precision on the Test dataset ($P_{te}$) of \fga and \efga for M-DNN1.}
    \footnotesize
    \begin{tabular}{l | r r r | r r r | r r r}
        \toprule
        \multirow{2}{5mm}{\textbf{Feature}}   &\multicolumn{3}{c|}{\textbf{\fga}}  &\multicolumn{3}{c|}{\textbf{\efga}}&\multicolumn{3}{c}{\textbf{Differences}}\\
          \cmidrule{2-10}
            & $\mathbf{R_{tr}}$    & $\mathbf{P_{te}}$  & $\mathbf{R_{te}}$ & $\mathbf{R_{tr}}$    & $\mathbf{P_{te}}$    & $\mathbf{R_{te}}$ & $\mathbf{R_{tr}}$    & $\mathbf{P_{te}}$    & $\mathbf{R_{te}}$\\
        \midrule
        Digit 0 & 82.54 & 99.64 & 85.51 & 95.36 & 98.30 & 94.18 & 12.82 & -1.34 & 8.67 \\
        Digit 1 & \cellcolor{blue!25}87.66 & 99.90 & \cellcolor{blue!25}88.99 & \cellcolor{blue!25}96.99 & 99.10 & \cellcolor{blue!25}97.44 & \cellcolor{orange!25}9.33 & -0.80 & \cellcolor{orange!25}8.45 \\
        Digit 2 & 64.99 & 99.57 & 67.73 & 91.84 & 99.04 & 89.73 & 26.85 & -0.53 & 22.00 \\
        Digit 3 & 69.48 & \cellcolor{blue!25}100.0 & 70.40 & 88.57 & 98.89 & 87.82 & 19.09 & -1.11 & 17.42 \\
        Digit 4 & 73.52 & 99.73 & 74.64 & 90.98 & 99.33 & 90.33 & 17.46 & -0.40 & 15.69 \\
        Digit 5 & 70.60 & \cellcolor{blue!25}100.0 & 72.76 & 89.98 & 99.50 & 89.24 & 19.38 & -0.50 & 16.48 \\
        Digit 6 & 75.65 & 99.58 & 74.74 & 91.87 & 98.38 & 88.73 & 16.22 & -1.20 & 13.99 \\
        Digit 7 & 48.27 & \cellcolor{blue!25}100.0 & 51.17 & 86.75 & \cellcolor{orange!25}98.14 & 82.30 & 38.48 & \cellcolor{orange!25}-1.86 & 31.13 \\
        Digit 8 & 53.89 & 99.81 & 54.83 & 85.28 & 98.50 & 80.90 & 31.39 & -1.31 & 26.07 \\
        Digit 9 & 54.24 & 99.31 & 56.79 & 78.45 & 98.49 & 77.70 & 24.21 & -0.82 & 20.91 \\
        2 and 7 & 39.19 & 99.64 & 39.76 & 80.15 & 99.23 & 80.83 & 40.96 & -0.41 & 41.07 \\
        9 and 6 & 35.27 & \cellcolor{orange!25}99.26 & 34.16 & 80.59 & 98.60 & 78.85 & 45.32 & -0.66 & 44.69 \\
        Line & \cellcolor{orange!25}27.81 & \cellcolor{blue!25}100.0 & \cellcolor{orange!25}28.11 & 86.52 & \cellcolor{blue!25}99.67 & 85.41 & \cellcolor{blue!25}58.71 & \cellcolor{blue!25}-0.33 & \cellcolor{blue!25}57.30 \\
        Circle & 32.83 & 99.63 & 34.10 & \cellcolor{orange!25}71.75 & 98.41 & \cellcolor{orange!25}70.93 & 38.92 & -1.22 & 36.83 \\
        \midrule
        Average & 58.28 & 99.72 & 59.55 & 86.79 & 98.83 & 85.31 & 28.51 & -0.89 & 25.76 \\
        \bottomrule
        \end{tabular}
    \label{tab:rq2-mnist}
\end{table*}

\begin{table*}[t]
    \centering
    \caption{Train ($R_{tr}$) and Test ($R_{te}$) Recall and the Precision on the Test dataset ($P_{te}$) of \fga and \efga for L-DNN1.}
    
    \begin{tabular}{l | r r r | r r r | r r r}
        \toprule
        \multirow{2}{5mm}{\textbf{Feature}}   &\multicolumn{3}{c|}{\textbf{\fga}}  &\multicolumn{3}{c|}{\textbf{\efga}}&\multicolumn{3}{c}{\textbf{Differences}}\\
          \cmidrule{2-10}
            & $\mathbf{R_{tr}}$    & $\mathbf{P_{te}}$  & $\mathbf{R_{te}}$ & $\mathbf{R_{tr}}$    & $\mathbf{P_{te}}$    & $\mathbf{R_{te}}$ & $\mathbf{R_{tr}}$    & $\mathbf{P_{te}}$    & $\mathbf{R_{te}}$\\
        \midrule
        CLL & \cellcolor{blue!25}65.63 & 99.13 & \cellcolor{blue!25}61.94 & 89.86 & 98.42 & 86.09 & 24.23 & -0.71 & 24.15 \\
        FL & 62.71 & \cellcolor{orange!25}99.07 & 56.17 & \cellcolor{orange!25}86.04 & \cellcolor{orange!25}97.54 & \cellcolor{orange!25}79.59 & \cellcolor{orange!25}23.33 & \cellcolor{blue!25}-1.53 & \cellcolor{orange!25}23.42 \\
        MCL & 60.05 & 99.94 & 56.80 & \cellcolor{blue!25}91.57 & 99.47 & \cellcolor{blue!25}87.03 & 31.52 & -0.47 & 30.23 \\
        CLL \& FL & 51.31 & 99.46 & 55.73 & 87.05 & 98.84 & 85.32 & 35.74 & -0.62 & 29.59 \\
        MCL \& FL & \cellcolor{orange!25}42.11 & \cellcolor{blue!25}99.99 & \cellcolor{orange!25}39.32 & 89.10 & \cellcolor{blue!25}99.81 & 83.86 & \cellcolor{blue!25}46.99 & \cellcolor{orange!25}-0.18 & \cellcolor{blue!25}44.54 \\
        CLL \& MCL & 53.33 & 99.98 & 53.22 & 90.43 & 99.37 & 86.16 & 37.10 & -0.61 & 32.94 \\
        \midrule
        Average & 55.86 & 99.59 & 53.86 & 89.01 & 98.91 & 84.68 & 33.15 & -0.69 & 30.81 \\
        \bottomrule
    \end{tabular}
    \label{tab:rq2-l-dnn1}
\end{table*}

\begin{table*}[t]
    \centering
    \caption{Train ($R_{tr}$) and Test ($R_{te}$) Recall and the Precision on the Test dataset ($P_{te}$) of \fga and \efga for L-DNN2.}
    \begin{tabular}{l | r r r | r r r | r r r}
        \toprule
        \multirow{2}{5mm}{\textbf{Feature}}   &\multicolumn{3}{c|}{\textbf{\fga}}  &\multicolumn{3}{c|}{\textbf{\efga}}&\multicolumn{3}{c}{\textbf{Differences}}\\
          \cmidrule{2-10}
            & $\mathbf{R_{tr}}$    & $\mathbf{P_{te}}$  & $\mathbf{R_{te}}$ & $\mathbf{R_{tr}}$    & $\mathbf{P_{te}}$    & $\mathbf{R_{te}}$ & $\mathbf{R_{tr}}$    & $\mathbf{P_{te}}$    & $\mathbf{R_{te}}$\\
        \midrule
        CLL & 94.02 & 99.57 & 97.76 & 100.0 & \cellcolor{orange!25}93.51 & 99.88 & 5.98 & \cellcolor{orange!25}-6.06 & 2.12 \\
        FL & 89.30 & 99.78 & 85.19 & 100.0 & 99.69 & \cellcolor{orange!25}95.72 & 10.70 & -0.09 & 10.53 \\
        MCL & \cellcolor{blue!25}99.92 & 99.90 & \cellcolor{blue!25}99.57 & 100.0 & 99.89 & 99.59 & \cellcolor{orange!25}0.08 & \cellcolor{blue!25}-0.01 & \cellcolor{orange!25}0.02 \\
        CLL \& FL & \cellcolor{orange!25}76.36 & 99.87 & \cellcolor{orange!25}77.59 & 100.0 & 99.83 & \cellcolor{blue!25}99.96 & \cellcolor{blue!25}23.64 & -0.04 & \cellcolor{blue!25}22.37 \\
        MCL \& FL & 88.17 & \cellcolor{blue!25}99.99 & 85.80 & 100.0 & \cellcolor{blue!25}99.96 & 97.43 & 11.83 & -0.03 & 11.63 \\
        CLL \& MCL & 86.36 & \cellcolor{orange!25}98.66 & 87.11 & 100.0 & 96.54 & 99.75 & 13.64 & -2.12 & 12.64 \\
        \midrule
        Average & 89.02 & 99.63 & 88.84 & 100.0 & 98.24 & 98.72 & 10.98 & -1.39 & 9.88 \\
        \bottomrule
    \end{tabular}
    \label{tab:rq2-l-dnn2}
\end{table*}

\begin{table*}[t]
    \centering
    \caption{Train ($R_{tr}$) and Test ($R_{te}$) Recall and the Precision on the Test dataset ($P_{te}$) of \fga for M-DNN3 and \efga for M-DNN1.}
    
    \begin{tabular}{l | r r r | r r r | r r r}
        \toprule
        \multirow{2}{5mm}{\textbf{Feature}}   &\multicolumn{3}{c|}{\textbf{\fga ~-- M-DNN3}}  &\multicolumn{3}{c|}{\textbf{\efga ~-- M-DNN1}}&\multicolumn{3}{c}{\textbf{Differences}}\\
          \cmidrule{2-10}
            & $\mathbf{R_{tr}}$    & $\mathbf{P_{te}}$  & $\mathbf{R_{te}}$ & $\mathbf{R_{tr}}$    & $\mathbf{P_{te}}$    & $\mathbf{R_{te}}$ & $\mathbf{R_{tr}}$    & $\mathbf{P_{te}}$    & $\mathbf{R_{te}}$\\
        \midrule
        Digit 0&86.11&99.66&89.76&95.36&98.30&94.18&9.25&-1.36&4.42\\
        Digit 1&93.37&99.91&94.87&\cellcolor{blue!25}96.99&99.10&\cellcolor{blue!25}97.44&3.62&-0.81&2.57\\
        Digit 2&\cellcolor{orange!25}67.56&99.32&\cellcolor{orange!25}70.89&91.84&99.04&89.73&\cellcolor{blue!25}24.28&-0.28&\cellcolor{blue!25}18.84\\
        Digit 3&82.47&\cellcolor{orange!25}98.94&83.76&88.57&98.89&87.82&6.10&-0.05&4.06\\
        Digit 4&\cellcolor{blue!25}94.81&99.36&\cellcolor{blue!25}95.47&90.98&99.33&90.33&-3.83&-0.03&-5.14\\
        Digit 5&85.53&99.21&85.70&89.98&99.50&89.24&4.45&\cellcolor{blue!25}0.29&3.54\\
        Digit 6&79.37&99.74&81.55&91.87&98.38&88.73&12.50&-1.36&7.18\\
        Digit 7&93.14&99.37&93.11&86.75&\cellcolor{orange!25}98.14&82.30&-6.39&-1.23&-10.81\\
        Digit 8&75.16&\cellcolor{blue!25}100.00&76.59&85.28&98.50&80.90&10.12&\cellcolor{orange!25}-1.50&4.31\\
        Digit 9&92.77&99.35&92.66&78.45&98.49&77.70&\cellcolor{orange!25}-14.32&-0.86&\cellcolor{orange!25}-14.96\\
        2 and 7&79.62&99.35&81.94&80.15&99.23&80.83&0.53&-0.12&-1.11\\
        6 and 9&83.17&99.68&81.36&80.59&98.60&78.85&-2.58&-1.08&-2.51\\
        Line & 91.32&99.62&92.27&86.52&\cellcolor{blue!25}99.67&85.41&-4.80&0.05&-6.86\\
        Circle &76.51&99.34&77.63&\cellcolor{orange!25}71.75&98.41&\cellcolor{orange!25}70.93&-4.76&-0.93&-6.70\\
        \midrule
        Average&84.35&99.49&85.54&86.79&98.83&85.31&2.44&-0.66&-0.23\\
        \bottomrule
    \end{tabular}
    \label{tab:rq2-m-dnn3}
\end{table*}
\textbf{Results - Exp1}.
\Cref{tab:rq2-mnist}, \Cref{tab:rq2-l-dnn1}, and \Cref{tab:rq2-l-dnn2} show the train ($R_{tr}$) and test ($R_{te}$) recall and the precision on the test dataset ($P_{te}$) of the rules computed by \fga and \efga for each feature of our and the M-DNN1, L-DNN1, and L-DNN2.
The tables also show the difference in the effectiveness of \fga and \efga.
The cells with the highest absolute value for a column are highlighted with background color blue, while cells with the lowest absolute value for a column are highlighted with background color orange.
We discuss the results related to each metric in the following.

\emph{Train Recall}. On M-DNN1, for \fga, the rules for the feature ``Digit 1'' and ``Line'' have the highest (87.66\%) and lowest (27.81\%) train recall. 
For \efga, the ensembles for the feature ``Digit 1'' and ``Circle'' have the highest (96.99\%) and lowest (71.75\%) train recall.
The rules for the feature ``Line'' and ``Digit 1'' have the highest (58.71\%) and lowest (9.33\%) absolute difference among train recall of \fga and \efga.
\fga and \efga have an average train recall of 58.28\% and 86.79\%. 
The average improvement is +28.51\% and reaches 58.71\% for some features.
On L-DNN1, \efga consistently outperforms \fga. The average train recall increases from 55.86\% (\fga) to 89.01\% (\efga), an improvement of +33.15\%. The feature ``MCL \& FL'' shows the largest increase in train recall (+46.99\%).
On L-DNN2, for \fga, the rules for the feature ``MCL'' and ``CLL \& FL'' have the highest (99.92\%) and lowest (76.36\%) train recall. For \efga, the ensembles for all features reach 100\% train recall. The average improvement is +10.98\%.
These results show that \efga considerably improves \fga for the train recall metric across all datasets.

\emph{Test Precision}. 
On M-DNN1, for \fga, the rules for the feature ``Digit 3'', ``Digit 5'', ``Digit 7'' and ``Line'' have the highest (100\%) test precision, while ''9 and 6'' have the lowest (99.26\%) one. 
For \efga the ensembles for the features ``Line'' and ``Digit 7'' have the highest (99.67\%) and the lowest (98.14\%) test precision.
The rules for the feature ``Digit 7'' and ``Line'' have the highest (1.86\%) and lowest (0.33\%) absolute difference among the test precision of \fga and \efga.
\fga and \efga have an average test precision of 99.72\% and 98.83\%. Therefore, \efga produces ensembles that are slightly less precise than the rules obtained with \fga.
The decrement in the average test precision is negligible: It is on average 0.89\% (\emph{min}=0.33, \emph{max}=1.86, \emph{std}=0.44).
On L-DNN1, \fga and \efga have an average test precision of 99.59\% and 98.91\%. The decrement is negligible (0.69\%).
Finally, on L-DNN2, \fga and \efga have an average test precision of 99.63\% and 98.24\%. The decrement is small (1.39\%), although slightly higher than L-DNN1, due to the drop in the ``CLL'' feature.
Therefore, \efga produces ensembles that are slightly less precise than the rules computed by \fga.

\emph{Test Recall}.
On M-DNN1, for \fga, the rules for the feature ``Digit 1'' and ``Line'' have the highest (88.99\%) and lowest (28.11\%) test recall. 
For \efga, the ensembles for the features ``Digit 1'' and ``Circle'' have the highest (97.44\%) and lowest (70.93\%) test recall.
The rules for the feature ``Line'' and ``Digit 1'' have the highest (57.30\%) and lowest (8.45\%) absolute difference among test recall of \fga and \efga.
\fga and \efga have an average test recall of 59.55\% and 85.31\%. 
This result shows that \efga significantly improves \fga for the test recall metric. 
The average improvement is +25.76\% and reaches 57.30\% for some features.
On L-DNN1, the average test recall increases significantly from 53.86\% to 84.68\% (+30.81\%). The rules for the feature ``MCL \& FL'' and ``FL'' have the highest (44.54\%) and lowest (23.42\%) absolute difference among test recall of \fga and \efga.
On L-DNN2, \fga and \efga have an average test recall of 88.84\% and 98.72\%. The average improvement is +9.88\%.
These results confirm that \efga significantly improves \fga for the test recall metric in all considered scenarios.

\textbf{Results - Exp2}. \Cref{tab:rq2-m-dnn3} reports the results for \fga on M-DNN3 and \efga on M-DNN1. We colored the cells with the minimum (orange) and maximum (blue) absolute difference.

\emph{Train Recall}. \fga and \efga have a comparable train recall: 
The average recall is respectively 84.35\% (Max: 94.81\%, Min: 67.56\%, StdDev: 8.19\%) and 86.79\% (Max: 96.99\%, Min: 71.75\%, StdDev: 6.76\%).
\efga on M-DNN1 produces (on average) results with an higher (+2.44\%) train recall than \fga on M-DNN3.
The rules for the feature ``Digit 2'' and ``Digit 2 and 7'' have the highest (24.28\%) and lowest (0.53\%) absolute difference among train recall of \fga and \efga.
This result shows that when the train recall of \fga is low, \efga can produce results comparable to a more suitable DNN for the \fga approach.

\emph{Test Precision}. 
\fga and \efga have a comparable test precision: 
The average precision is respectively 99.49\% (Max: 100.00\%, Min: 98.94\%, StdDev: 0.29\%) and 98.83\% (Max: 99.67\%, Min: 98.14\%, StdDev: 0.47\%).
\efga on M-DNN1 produces (on average) results with a lower (-0.66\%) test precision than \fga on M-DNN3.
The rules for features ``Digit 8'' and ``Digit 4'' have the highest (1.50\%) and lowest (0.03\%) absolute difference among the \fga and \efga test precision.
This result shows that considering the \efga ensemble-based approach does not significantly affect the test precision.

\emph{Test Recall}.  \fga and \efga have a comparable test recall: 
The average recall is respectively 85.54\% (Max: 95.47\%, Min: 70.89\%, StdDev: 7.66\%) and 85.31\% (Max: 97.44\%, Min: 70.93\%, StdDev: 6.86\%).
\efga on M-DNN1 produces (on average) results with a lower (-0.23\%) test recall than \fga on M-DNN3.
The rules for the feature ``Digit 2'' and ``Digit 2 and 7'' have the highest (18.84\%) and lowest (1.11\%) absolute difference among test recall of \fga and \efga.
This result shows that considering the \efga ensemble-based approach does not significantly affect the test recall.

\begin{Answer}[RQ2---Effectiveness of \efga]
Our results (M-DNN1 and L-DNN1) show that \efga produce ensembles with considerably higher train (up to +33.15\%) and test (up to +30.81\%) recall, and negligible reductions (avg. max -1.39\%) in test precision than the rules computed by \fga.
For L-DNN2, the improvements are smaller due to the higher recall of the original \fga rules.
When the train and test recall of \fga is insufficient,  \efga can produce results comparable to cases where \fga has shown to be effective.
\end{Answer}

\section{Discussion and Threats to Validity}
\label{sec:discussion}

\efga requires engineers to select the layer at which the ensemble of rules is created. 
Alternatively, \efga could be instrumented to create ensembles of rules at every layer, and to subsequently select the one with the highest recall. 

The results from RQ1 show that the length of the ensemble increases as its test recall, and that increasing the test recall of the ensemble reduces its test precision.
However, our results on the \Lymphoma benchmark (e.g., L-DNN2) highlight that the cost in terms of length is strictly related to the performance of the underlying DNN: for networks that already perform well, the ensembles remain compact even at very high recall levels.
Furthermore, it also highlights the difference in the effectiveness of our criteria: Depending on the user's preference, a different criterion can be selected.
We also notice that (i) the recall of all the criteria are better than \texttt{TOP(1)} (which corresponds to \fga). This result is expected since each criterion adds rules to \texttt{TOP(1)}, and according to \Cref{theor:recallComp} this action increases the recall of the ensemble, and (b) Most criteria are located on the Pareto front, offering a wider choice of rules that provide alternative tradeoffs between precision and recall.

The results from RQ2 show that \efga produces a significant improvement in train and test recall with a negligible reduction in test precision. 
This result show that, \efga can identify a higher number of input images with a specific feature compared to \fga.
The features that showed the lowest train recall with \fga benefit the most from  \efga. 
These findings are confirmed by our experiments on the \Lymphoma dataset, demonstrating that \efga is effective across different domains. Moreover, the results on L-DNN2 show that \efga allows reaching near-perfect recall (100\%) even when the baseline \fga performance is already high, demonstrating its utility in refining high-performing models.
The comparison between \fga applied to M-DNN1 and \efga applied to M-DNN3 suggests that \efga over M-DNN1 produces (on average) an ensemble of rules with performances comparable with those obtained for \fga applied to the best DNN reported in \cite{FGA_Replication}. 
This demonstrates that the user can apply the approach to DNN architectures that produce individual rules with low recall, obtaining performances comparable to those of other DNNs. 
Note that the differences for individual features (in terms of performances) can vary greatly between the performances of the top rule obtained for \fga over M-DNN3 and \efga over M-DNN1.
Our results show that this limitation is successfully addressed by \efga. 

\emph{Threats to Validity}.
The choice of the \Mnist dataset can threaten the \emph{external validity} of our results.
We mitigated this threat by evaluating \efga on two distinct benchmarks: \Mnist and \Lymphoma.
\Mnist is a popular dataset used in the ML field, allowing comparison with previous studies.
The choice of the layer and the selection criterion can threaten the \emph{internal validity} of our results since considering a different layer could lead to different results. 
To mitigate this threat, we proposed a list of criteria covering alternative scenarios.
Our list is not exhaustive, and other ad-hoc criteria can be considered. \section{Related Works}
\label{sec:related}
We present related works that applied an approach based on an ensemble of methods, models, or algorithms and we discuss other works focused on the explainability of DNN.

\emph{Ensembles-based approaches}. The aggregation of models or rules to increase performances or explainability of the models has been extensively considered in the fields of machine learning and software engineering (\cite{liu2015,chen2022, moussa2022,zhong2024,yang25,lu22}).
Moussa et al.~\cite{moussa2022} propose an ensemble of ML models to predict the presence of defects in software.
Chen et al.~\cite{chen2022} propose an ensemble of two models to address the presence of human biases in ML algorithms. Liu et al.~\cite{liu2015} propose an ensemble of rules-based classifiers, that combines models trained on different partitions of the dataset to reduce the effects of overfitting of individual models.
Our approach differs from the aforementioned ones since we do not aggregate different models, but we aggregate decision rules extracted from the same model.
Lu et al.~\cite{lu22} describe an ensemble of rules approach to increase the explainability of ML models for fraud detection. It combines rules from tree-based models with those extracted from a computational graph. Our approach differs since aggregate only rules extracted from a decision tree.

\emph{Analysis of DNNs}. Many works (\cite{zhou2018,kim2018,yeh2020,kusters2020,ghorbani2019,koh2020,chen2020concept,barbiero2022entropy}) belong to the area of concept-based explanation of DNNs, which was summarized in recent surveys (\cite{Lee25,Poeta2025}). According to the definitions provided in Poeta et al.~\cite{Poeta2025} (which is also adopted by Lee et al.~\cite{Lee25}), our work is close to symbolic concept explanation, which is defined as a human interpretable abstraction.
Of these works, Zhou et al.~\cite{zhou2018} describe an approach to decompose the prediction of a DNN into a set of interpretable concepts, how relevant each concept is to the final prediction. Some of the methods in the field of concept explanation tried to provide an approach that explains the predictions of a DNN at training time. For example, Chen et al.~\cite{chen2020concept} describes a method that normalizes the latent space of the DNN with respect to a given concept.

This work also shares similarities with DNN repair (i.e., the repair of neural networks), which has been addressed by several works \cite{Lyu2024,Mancu2024,calsi2023,Fahmy2022}.
These works differ from our as they aim to perform distributed repair, which re-trains the DNN, while the focus of \efga is on the interpretation of the model behavior.  
Lastly, Fahmy et al. \cite{fahmy2023} proposed an approach that identifies misclassified images, and generates a new dataset of similar ones for retraining. %
 \section{Conclusion and Future Works}
\label{sec:conclusion}
This paper introduced \efga, an extension of \fga that increases the recall of its rules by aggregating them into ensembles. 
We compared the performances of the ensembles generated by different criteria. Our results show that \efga configured with a proper criterion can increase significantly the recall compared to \fga, with a negligible reduction in precision. 
We also compared the results obtained with \fga for the best-performing DNN reported in \cite{FGA_Replication} with those of \efga for a different DNN. 
Our results show that \efga can produce ensembles of rules with performances comparable to those obtained by \fga for other DNN architectures.

\section*{Data Availability}
Our replication package is available online \url{https://figshare.com/account/articles/31691773} \cite{ReplicationPackage}.

\bibliographystyle{IEEEtran}

\begin{thebibliography}{10}
\providecommand{\url}[1]{#1}
\csname url@samestyle\endcsname
\providecommand{\newblock}{\relax}
\providecommand{\bibinfo}[2]{#2}
\providecommand{\BIBentrySTDinterwordspacing}{\spaceskip=0pt\relax}
\providecommand{\BIBentryALTinterwordstretchfactor}{4}
\providecommand{\BIBentryALTinterwordspacing}{\spaceskip=\fontdimen2\font plus
\BIBentryALTinterwordstretchfactor\fontdimen3\font minus
  \fontdimen4\font\relax}
\providecommand{\BIBforeignlanguage}[2]{{\expandafter\ifx\csname l@#1\endcsname\relax
\typeout{** WARNING: IEEEtran.bst: No hyphenation pattern has been}\typeout{** loaded for the language `#1'. Using the pattern for}\typeout{** the default language instead.}\else
\language=\csname l@#1\endcsname
\fi
#2}}
\providecommand{\BIBdecl}{\relax}
\BIBdecl

\bibitem{khan2022software}
\BIBentryALTinterwordspacing
M.~A. Khan, N.~S. Elmitwally, S.~Abbas, S.~Aftab, M.~Ahmad, M.~Fayaz, and
  F.~Khan, ``Software defect prediction using artificial neural networks: A
  systematic literature review,'' \emph{Scientific Programming}, p. 2117339,
  2022. [Online]. Available: \url{https://doi.org/10.1155/2022/2117339}
\BIBentrySTDinterwordspacing

\bibitem{boujida2024neural}
\BIBentryALTinterwordspacing
F.~E. Boujida, F.~A. Amazal, and A.~Idri, ``Neural networks-based software
  development effort estimation: A systematic literature review,''
  \emph{Journal of Software: Evolution and Process}, vol.~37, no.~2, p. e2756,
  2024. [Online]. Available: \url{https://doi.org/10.1002/smr.2756}
\BIBentrySTDinterwordspacing

\bibitem{1634649}
S.-J. Han and S.-B. Cho, ``Evolutionary neural networks for anomaly detection
  based on the behavior of a program,'' \emph{IEEE Transactions on Systems,
  Man, and Cybernetics, Part B (Cybernetics)}, vol.~36, no.~3, pp. 559--570,
  2006.

\bibitem{baier2019challenges}
L.~Baier, F.~J{\"o}hren, and S.~Seebacher, ``Challenges in the deployment and
  operation of machine learning in practice.'' in \emph{European Conference on
  Information Systems - ECIS}, vol.~1, 2019.

\bibitem{janowczyk2016}
A.~Janowczyk and A.~Madabhushi, ``Deep learning for digital pathology image
  analysis: A comprehensive tutorial with selected use cases,'' \emph{Journal
  of Pathology Informatics}, vol.~7, no.~1, p.~29, 2016.

\bibitem{molnar2022interpretable}
C.~Molnar, \emph{Interpretable Machine Learning: A Guide for Making Black Box
  Models Explainable}.\hskip 1em plus 0.5em minus 0.4em\relax Indipendently
  published, 2022.

\bibitem{kim2018}
B.~Kim, M.~Wattenberg, J.~Gilmer, C.~Cai, J.~Wexler, F.~Viegas, and R.~sayres,
  ``Interpretability beyond feature attribution: Quantitative testing with
  concept activation vectors ({TCAV}),'' in \emph{Proceedings of the 35th
  International Conference on Machine Learning}, ser. Proceedings of Machine
  Learning Research, J.~Dy and A.~Krause, Eds., vol.~80.\hskip 1em plus 0.5em
  minus 0.4em\relax PMLR, 10--15 Jul 2018, pp. 2668--2677.

\bibitem{yeh2020}
\BIBentryALTinterwordspacing
C.-K. Yeh, B.~Kim, S.~Arik, C.-L. Li, T.~Pfister, and P.~Ravikumar, ``On
  completeness-aware concept-based explanations in deep neural networks,'' in
  \emph{Advances in Neural Information Processing Systems}, H.~Larochelle,
  M.~Ranzato, R.~Hadsell, M.~Balcan, and H.~Lin, Eds., vol.~33.\hskip 1em plus
  0.5em minus 0.4em\relax Curran Associates, Inc., 2020, pp. 20\,554--20\,565.
  [Online]. Available:
  \url{https://proceedings.neurips.cc/paper_files/paper/2020/file/ecb287ff763c169694f682af52c1f309-Paper.pdf}
\BIBentrySTDinterwordspacing

\bibitem{kusters2020}
F.~Küsters, P.~Schichtel, S.~Ahmed, and A.~Dengel, ``Conceptual explanations
  of neural network prediction for time series,'' in \emph{2020 International
  Joint Conference on Neural Networks (IJCNN)}, 2020, pp. 1--6.

\bibitem{ghorbani2019}
\BIBentryALTinterwordspacing
A.~Ghorbani, J.~Wexler, J.~Y. Zou, and B.~Kim, ``Towards automatic
  concept-based explanations,'' in \emph{Advances in Neural Information
  Processing Systems}, H.~Wallach, H.~Larochelle, A.~Beygelzimer,
  F.~d\textquotesingle Alch\'{e}-Buc, E.~Fox, and R.~Garnett, Eds.,
  vol.~32.\hskip 1em plus 0.5em minus 0.4em\relax Curran Associates, Inc.,
  2019. [Online]. Available:
  \url{https://proceedings.neurips.cc/paper_files/paper/2019/file/77d2afcb31f6493e350fca61764efb9a-Paper.pdf}
\BIBentrySTDinterwordspacing

\bibitem{koh2020}
\BIBentryALTinterwordspacing
P.~W. Koh, T.~Nguyen, Y.~S. Tang, S.~Mussmann, E.~Pierson, B.~Kim, and
  P.~Liang, ``Concept bottleneck models,'' in \emph{Proceedings of the 37th
  International Conference on Machine Learning}, ser. Proceedings of Machine
  Learning Research, H.~D. III and A.~Singh, Eds., vol. 119.\hskip 1em plus
  0.5em minus 0.4em\relax PMLR, 13--18 Jul 2020, pp. 5338--5348. [Online].
  Available: \url{https://proceedings.mlr.press/v119/koh20a.html}
\BIBentrySTDinterwordspacing

\bibitem{chen2020concept}
Z.~Chen, Y.~Bei, and C.~Rudin, ``Concept whitening for interpretable image
  recognition,'' \emph{Nature Machine Intelligence}, vol.~2, no.~12, pp.
  772--782, 2020.

\bibitem{barbiero2022entropy}
P.~Barbiero, G.~Ciravegna, F.~Giannini, P.~Li{\'o}, M.~Gori, and S.~Melacci,
  ``Entropy-based logic explanations of neural networks,'' in \emph{Proceedings
  of the AAAI Conference on Artificial Intelligence}, vol.~36, no.~6, 2022, pp.
  6046--6054.

\bibitem{Gopinath_2023}
D.~Gopinath, L.~Lungeanu, R.~Mangal, C.~P{\u{a}}s{\u{a}}reanu, S.~Xie, and
  H.~Yu, ``Feature-guided analysis of neural networks,'' in \emph{Fundamental
  Approaches to Software Engineering}, 2023, pp. 133--142.

\bibitem{Beland_2020}
S.~Beland, I.~Chang, A.~Chen, M.~Moser, J.~Paunicka, D.~Stuart, J.~Vian,
  C.~Westover, and H.~Yu, ``Towards assurance evaluation of autonomous
  systems,'' in \emph{Proceedings of the 39th International Conference on
  Computer-Aided Design}, ser. ICCAD '20.\hskip 1em plus 0.5em minus
  0.4em\relax New York, NY, USA: Association for Computing Machinery, 2020.

\bibitem{Frew_2004}
E.~Frew, T.~McGee, Z.~Kim, X.~Xiao, S.~Jackson, M.~Morimoto, S.~Rathinam,
  J.~Padial, and R.~Sengupta, ``Vision-based road-following using a small
  autonomous aircraft,'' in \emph{2004 IEEE Aerospace Conference Proceedings
  (IEEE Cat. No.04TH8720)}, vol.~5, 2004, pp. 3006--3015.

\bibitem{caesar2020nuscenes}
H.~Caesar, V.~Bankiti, A.~H. Lang, S.~Vora, V.~E. Liong, Q.~Xu, A.~Krishnan,
  Y.~Pan, G.~Baldan, and O.~Beijbom, ``nuscenes: A multimodal dataset for
  autonomous driving,'' in \emph{Proceedings of the IEEE/CVF conference on
  computer vision and pattern recognition}, 2020, pp. 11\,621--11\,631.

\bibitem{FGA_Replication}
\BIBentryALTinterwordspacing
F.~Formica, S.~Gregis, A.~F. Zanenga, A.~Rota, M.~Lawford, and C.~Menghi,
  ``Feature-guided analysis of neural networks: A replication study,'' 2025.
  [Online]. Available: \url{https://arxiv.org/abs/2511.00052}
\BIBentrySTDinterwordspacing

\bibitem{lecun1998}
Y.~Lecun, L.~Bottou, Y.~Bengio, and P.~Haffner, ``Gradient-based learning
  applied to document recognition,'' \emph{Proceedings of the IEEE}, vol.~86,
  no.~11, pp. 2278--2324, 1998.

\bibitem{Gopinath_2019}
D.~Gopinath, H.~Converse, C.~P{\u{a}}s{\u{a}}reanu, and A.~Taly, ``Property
  inference for deep neural networks,'' in \emph{2019 34th IEEE/ACM
  International Conference on Automated Software Engineering (ASE)}, 2019, pp.
  797--809.

\bibitem{lymphomacaffe2016}
``{Pretrained Caffe Models for Lymphoma Dataset},''
  \url{https://github.com/choosehappy/public/tree/master/DL\%20tutorial\%20Code/7-lymphoma/models},
  2016, last accessed in Sep 2025.

\bibitem{liu2015}
H.~Liu, A.~Gegov, and M.~Cocea, ``Hybrid ensemble learning approach for
  generation of classification rules,'' in \emph{2015 International Conference
  on Machine Learning and Cybernetics (ICMLC)}, vol.~1, 2015, pp. 377--382.

\bibitem{chen2022}
\BIBentryALTinterwordspacing
Z.~Chen, J.~M. Zhang, F.~Sarro, and M.~Harman, ``Maat: a novel ensemble
  approach to addressing fairness and performance bugs for machine learning
  software,'' in \emph{Proceedings of the 30th ACM Joint European Software
  Engineering Conference and Symposium on the Foundations of Software
  Engineering}, ser. ESEC/FSE 2022.\hskip 1em plus 0.5em minus 0.4em\relax New
  York, NY, USA: Association for Computing Machinery, 2022, pp. 1122--1134.
  [Online]. Available: \url{https://doi.org/10.1145/3540250.3549093}
\BIBentrySTDinterwordspacing

\bibitem{moussa2022}
\BIBentryALTinterwordspacing
R.~Moussa, G.~Guizzo, and F.~Sarro, ``Meg: Multi-objective ensemble generation
  for software defect prediction,'' in \emph{Proceedings of the 16th ACM / IEEE
  International Symposium on Empirical Software Engineering and Measurement},
  ser. ESEM '22.\hskip 1em plus 0.5em minus 0.4em\relax New York, NY, USA:
  Association for Computing Machinery, 2022, pp. 159--170. [Online]. Available:
  \url{https://doi.org/10.1145/3544902.3546255}
\BIBentrySTDinterwordspacing

\bibitem{zhong2024}
\BIBentryALTinterwordspacing
W.~Zhong, C.~Li, K.~Liu, T.~Xu, J.~Ge, T.~F. Bissyande, B.~Luo, and V.~Ng,
  ``Practical program repair via preference-based ensemble strategy,'' in
  \emph{Proceedings of the IEEE/ACM 46th International Conference on Software
  Engineering}, ser. ICSE '24.\hskip 1em plus 0.5em minus 0.4em\relax New York,
  NY, USA: Association for Computing Machinery, 2024. [Online]. Available:
  \url{https://doi.org/10.1145/3597503.3623310}
\BIBentrySTDinterwordspacing

\bibitem{yang25}
Y.~Yang, F.~Yang, Y.~Bai, and H.~Wang, ``Self-interpretable reinforcement
  learning via rule ensembles,'' in \emph{Proceedings of the 24th International
  Conference on Autonomous Agents and Multiagent Systems}, ser. AAMAS
  '25.\hskip 1em plus 0.5em minus 0.4em\relax Richland, SC: International
  Foundation for Autonomous Agents and Multiagent Systems, 2025, pp.
  2235--2243.

\bibitem{lu22}
\BIBentryALTinterwordspacing
L.~Yu, M.~Li, X.~Huang, W.~Zhu, Y.~Fang, J.~Zhou, and L.~Li, ``Metarule: A
  meta-path guided ensemble rule set learning for explainable fraud
  detection,'' in \emph{Proceedings of the 31st ACM International Conference on
  Information \& Knowledge Management}, ser. CIKM '22.\hskip 1em plus 0.5em
  minus 0.4em\relax New York, NY, USA: Association for Computing Machinery,
  2022, pp. 4650--4654. [Online]. Available:
  \url{https://doi.org/10.1145/3511808.3557641}
\BIBentrySTDinterwordspacing

\bibitem{zhou2018}
B.~Zhou, Y.~Sun, D.~Bau, and A.~Torralba, ``Interpretable basis decomposition
  for visual explanation,'' in \emph{Computer Vision -- ECCV 2018}, V.~Ferrari,
  M.~Hebert, C.~Sminchisescu, and Y.~Weiss, Eds.\hskip 1em plus 0.5em minus
  0.4em\relax Cham: Springer International Publishing, 2018, pp. 122--138.

\bibitem{Lee25}
J.~H. Lee, G.~Mikriukov, G.~Schwalbe, S.~Wermter, and D.~Wolter,
  ``Concept-based explanations in computer vision: Where are we and where
  could we go?'' in \emph{Computer Vision -- ECCV 2024 Workshops}, A.~Del~Bue,
  C.~Canton, J.~Pont-Tuset, and T.~Tommasi, Eds.\hskip 1em plus 0.5em minus
  0.4em\relax Cham: Springer Nature Switzerland, 2025, pp. 266--287.

\bibitem{Poeta2025}
\BIBentryALTinterwordspacing
E.~Poeta, G.~Ciravegna, E.~Pastor, T.~Cerquitelli, and E.~Baralis,
  ``Concept-based explainable artificial intelligence: A survey,'' \emph{ACM
  Comput. Surv.}, Nov. 2025, just Accepted. [Online]. Available:
  \url{https://doi.org/10.1145/3774643}
\BIBentrySTDinterwordspacing

\bibitem{Lyu2024}
\BIBentryALTinterwordspacing
D.~Lyu, Z.~Zhang, P.~Arcaini, F.~Ishikawa, T.~Laurent, and J.~Zhao,
  ``Search-based repair of dnn controllers of ai-enabled cyber-physical systems
  guided by system-level specifications,'' in \emph{Proceedings of the Genetic
  and Evolutionary Computation Conference}, ser. GECCO '24.\hskip 1em plus
  0.5em minus 0.4em\relax New York, NY, USA: Association for Computing
  Machinery, 2024, pp. 1435--1444. [Online]. Available:
  \url{https://doi.org/10.1145/3638529.3654078}
\BIBentrySTDinterwordspacing

\bibitem{Mancu2024}
\BIBentryALTinterwordspacing
A.~Mancu, T.~Laurent, F.~Rieger, P.~Arcaini, F.~Ishikawa, and D.~Rueckert,
  ``More is not always better: Exploring early repair of dnns,'' in
  \emph{Proceedings of the 5th IEEE/ACM International Workshop on Deep Learning
  for Testing and Testing for Deep Learning}, ser. DeepTest '24.\hskip 1em plus
  0.5em minus 0.4em\relax New York, NY, USA: Association for Computing
  Machinery, 2024, pp. 13--16. [Online]. Available:
  \url{https://doi.org/10.1145/3643786.3648024}
\BIBentrySTDinterwordspacing

\bibitem{calsi2023}
D.~L. Calsi, M.~Duran, X.-Y. Zhang, P.~Arcaini, and F.~Ishikawa, ``Distributed
  repair of deep neural networks,'' in \emph{2023 IEEE Conference on Software
  Testing, Verification and Validation (ICST)}, 2023, pp. 83--94.

\bibitem{Fahmy2022}
\BIBentryALTinterwordspacing
H.~Fahmy, F.~Pastore, and L.~Briand, ``Hudd: a tool to debug dnns for safety
  analysis,'' in \emph{Proceedings of the ACM/IEEE 44th International
  Conference on Software Engineering: Companion Proceedings}, ser. ICSE
  '22.\hskip 1em plus 0.5em minus 0.4em\relax New York, NY, USA: Association
  for Computing Machinery, 2022, pp. 100--104. [Online]. Available:
  \url{https://doi.org/10.1145/3510454.3516858}
\BIBentrySTDinterwordspacing

\bibitem{fahmy2023}
\BIBentryALTinterwordspacing
H.~Fahmy, F.~Pastore, L.~Briand, and T.~Stifter, ``Simulator-based explanation
  and debugging of hazard-triggering events in dnn-based safety-critical
  systems,'' \emph{ACM Trans. Softw. Eng. Methodol.}, vol.~32, no.~4, May 2023.
  [Online]. Available: \url{https://doi.org/10.1145/3569935}
\BIBentrySTDinterwordspacing

\bibitem{ReplicationPackage}
F.~Formica, S.~Gregis, A.~Rota, A.~F. Zanenga, D.~Gopinath, M.~Lawford, and
  C.~Menghi, ``{Replication Package for ``Ensemble-based Feature Guided
  Analysis''},'' \url{https://figshare.com/account/articles/31691773}, 2026,
  last accessed in Mar 2026.

\end{thebibliography}

\end{document}